\documentclass{article}

\usepackage{microtype}
\usepackage{graphicx}
\usepackage{booktabs} 

\usepackage{hyperref}

\usepackage[accepted]{icml2020}

\icmltitlerunning{Stronger and Faster Wasserstein Adversarial Attacks}

\usepackage[ruled,linesnumbered,vlined,algo2e]{algorithm2e}
\usepackage{amsmath}
\usepackage{amsthm}

\usepackage{cleveref}

\usepackage{multirow}
\usepackage{multicol}

\usepackage{thmtools}
\usepackage{thm-restate}

\usepackage{subcaption}
\usepackage{tikz}
\usepackage{tikzscale}
\usetikzlibrary{spy,calc}
\usetikzlibrary{decorations.pathreplacing}
\usetikzlibrary{patterns}

\usepackage{lipsum}

\usepackage{pifont}
\newcommand{\cmark}{\ding{51}}
\newcommand{\xmark}{\ding{55}}
\newcommand{\EE}{\mathds{E}}

\newcommand*{\textcite}{\citet}
\newcommand*{\parencite}{\citep}

\usepackage[utf8]{inputenc}
\usepackage[T1]{fontenc}

\usepackage[resetlabels]{multibib}
\newcites{add}{Additional References}

\usepackage{macro}

\usepackage{pgfplots}
\pgfplotsset{width=10cm,compat=1.9}

\newcommand{\getcorners}[1]{
\coordinate (#1_nw) at (#1.north west);
\coordinate (#1_ne) at (#1.north east);
\coordinate (#1_sw) at (#1.south west);
\coordinate (#1_se) at (#1.south east);
}

\newcommand{\drawtext}[4]{
\coordinate (rec_nw) at ($#1*(adv_imgs_nw)+(adv_imgs_ne)-#1*(adv_imgs_ne)$);
\coordinate (rec_ne) at ($#2*(adv_imgs_nw)+(adv_imgs_ne)-#2*(adv_imgs_ne)$);
\coordinate (rec_sw) at ($#1*(adv_imgs_sw)+(adv_imgs_se)-#1*(adv_imgs_se)$);
\coordinate (rec_se) at ($#2*(adv_imgs_sw)+(adv_imgs_se)-#2*(adv_imgs_se)$);

\draw[decorate,decoration={brace,mirror}] ($(rec_sw)+(0.03,0)$) -- ($(rec_se)-(0.03,0)$);

\node[inner sep=0] (xxx) at ($0.5*(rec_sw)+0.5*(rec_se)-0.25*(0,1)$) {#3};
\node[inner sep=0] (yyy) at ($0.5*(rec_sw)+0.5*(rec_se)-0.50*(0,1)$) {#4};
}

\begin{document}

\twocolumn[
\icmltitle{
Stronger and Faster Wasserstein Adversarial Attacks
}




\begin{icmlauthorlist}
\icmlauthor{Kaiwen Wu}{waterloo,vector}
\icmlauthor{Allen Houze Wang}{waterloo,vector}
\icmlauthor{Yaoliang Yu}{waterloo,vector}
\end{icmlauthorlist}

\icmlaffiliation{waterloo}{David R. Cheriton School of Computer Science, University of Waterloo}
\icmlaffiliation{vector}{Vector Institute}

\icmlcorrespondingauthor{Kaiwen Wu}{kaiwen.wu@uwaterloo.ca}
\vskip 0.3in
]



\printAffiliationsAndNotice{}  

\begin{abstract}

Deep models, while being extremely flexible and accurate, are surprisingly vulnerable to ``small, imperceptible'' perturbations known as adversarial attacks. While the majority of existing attacks focus on measuring perturbations under the $\ell_p$ metric, Wasserstein distance, which takes geometry in pixel space into account, has long been known to be a suitable metric for measuring image quality and has recently risen as a compelling alternative to the $\ell_p$ metric in adversarial attacks. However, constructing an effective attack under the Wasserstein metric is computationally much more challenging and calls for better optimization algorithms. We address this gap in two ways: (a) we develop an exact yet efficient projection operator to enable a stronger projected gradient attack; (b) we show that the Frank-Wolfe method equipped with a suitable linear minimization oracle works extremely fast under Wasserstein constraints. Our algorithms not only converge faster but also generate much stronger attacks. For instance, we decrease the accuracy of a residual network on CIFAR-10 to $3.4\%$ within a Wasserstein perturbation ball of radius $0.005$, in contrast to $65.6\%$ using the previous Wasserstein attack based on an \emph{approximate} projection operator. Furthermore, employing our stronger attacks in adversarial training significantly improves the robustness of adversarially trained models.

\end{abstract}

\section{Introduction}

Deep models are surprisingly vulnerable to adversarial attacks, namely small or even imperceptible perturbations that completely change the prediction \parencite{SzegedyZSBEGF14}.
The existence of adversarial examples has raised a lot of security concerns on deep models, and a substantial amount of work has devoted to this emerging field \parencite{GoodfellowMP18}, including various attacks as well as defences
\parencite[\eg][]{GoodfellowSS15,PapernotMWJS16,CarliniWagner17,Moosavi-DezfooliFF16,KurakinGB17,MadryMSTV18}.
Some empirical defences are shown ineffective later under stronger attacks \parencite{AthalyeCW18}, which has motivated a line of research on certified defences with \emph{provable} guarantees \parencite[\eg][]{WongKolter18,TjengXT19,GowalDSBQUAMK19,RaghunathanSL18,CohenRK19}.

\begin{figure}[t]
\centering
\includegraphics[width=\linewidth]{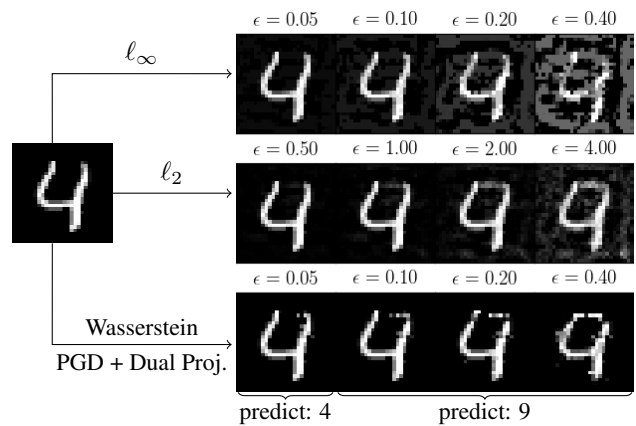}
\caption{$\ell_\infty$, $\ell_2$ and Wasserstein adversarial examples generated by projected gradient descent (PGD) with dual projection.
While $l_\infty$ and $l_2$ norm adversarial examples tend to perturb the background, Wasserstein adversarial examples \emph{redistribute} the pixel mass.
}
\label{fig:comparison_lp_wasserstein_adversarial}
\end{figure}

\begin{table*}[t]
\centering
\caption{A summary of projected Sinkhorn and our proposed algorithms. \textbf{3rd and 4th column:} Computational complexity for a single iteration of each algorithm (without or with local transportation constraint). $n$ is the dimension of inputs, and $k$ is the local transportation region size (see \S\ref{sec:local_transportation}). \textbf{5th column:} The exact convergence rate of projected Sinkhorn is not known yet. \textbf{Last column:} Whether the method is an exact or approximate algorithm.}
\vspace{-.4em}
\label{tb:summary}
\begin{tabular}{c | c @{\hspace{0.9\tabcolsep}}  c @{\hspace{0.9\tabcolsep}} c @{\hspace{0.99\tabcolsep}} c  @{\hspace{0.99\tabcolsep}} c}
\toprule
method & optimization space & cost/iter & cost/iter (local) & convergence rate & exact? \\
\midrule
projected Sinkhorn \parencite{WongSK19} & image space & $O(n^2)$ & $O(n k^2)$ & ? & \xmark \\
dual projection (ours) & coupling matrix & $O(n^2 \log n)$ & $O(n k^2\log k)$ & linear & \cmark \\ 
dual linear minimization oracle (ours) & coupling matrix & $O(n^2)$ & $O(n k^2)$ & linear & \xmark \\
\bottomrule
\end{tabular}
\vspace{-.4em}
\end{table*}

The majority of existing work on adversarial robustness focused on the $\ell_p$ threat model where the perturbation is measured using the $\ell_p$ norm.
However, the $\ell_p$ norm, despite being computationally convenient, is long known to be a poor proxy for measuring image similarity: two semantically similar images for human perception are not necessarily close under $\ell_p$ norm, see \parencite[\eg][]{WangBovik09} for some astonishing examples.
To this end, threat models beyond $\ell_p$ norms have been proposed, \eg \textcite{EngstromTTSM19} explore geometric transformation to fool deep networks; \citet{LaidlawFeizi19} use point-wise functions on pixel values to flip predictions; \textcite{TramerBoneh19} study robustness against multiple ($\ell_p$) perturbations.

In the same spirit, \textcite{WongSK19} recently proposed the Wasserstein threat model, \ie, adversarial examples are subject to a perturbation budget measured by the Wasserstein distance \parencite[a.k.a. earth mover's distance, see \eg][]{PeyreCuturi19}. The idea is to \emph{redistribute} pixel mass
instead of adjusting each pixel value as in previous $\ell_p$ threat models. Examples of Wasserstein adversarial attacks (generated by our algorithm) and comparison to $\ell_2$ and $\ell_\infty$ adversarial attacks are shown in \Cref{fig:comparison_lp_wasserstein_adversarial}. A key advantage of the former 
is that it explicitly captures geometric information in the image space (\ie how mass moves around matters). For example, a slight translation or rotation of an image usually induces small change in the Wasserstein distance, but may change the $l_p$ distance drastically. In addition, Wasserstein distance has played a pivotal role in generative adversarial networks \parencite{ArjovskyCB17},  computer vision \parencite{RubnerGT97}, and much beyond \parencite{PeyreCuturi19}.

\vspace{.12em}

\textbf{Contributions}  Generating Wasserstein adversarial examples requires solving a Wasserstein constrained optimization problem.
\textcite{WongSK19} developed a projected gradient attack using \emph{approximate} projection (projected Sinkhorn), which we find sometimes too crude and generating suboptimal attacks.
In this work, we develop two stronger and faster attacks, based on reformulating the optimization problem (\S\ref{sec:prob}) and applying projected gradient descent (PGD) and Frank-Wolfe (FW), respectively.
For the PGD attack, we design a specialized algorithm to compute the projection \emph{exactly}, which significantly improves the attack quality (\S\ref{sec:dual_proj}).
For FW, we develop a \emph{faster} algorithm to solve the linear minimization step with entropic smoothing (\S\ref{sec:frank_wolfe}).
Both subroutines enjoy fast linear convergence rates.
Synthetic experiments on simple convex functions (\Cref{tab:sinkhorn_proj_experiment}) show that both algorithms are able to converge to high precision solutions.
Extensive experiments on large scale datasets (\S\ref{sec:exp}) confirm the improved quality and speed of our attacks.
In particular, for the first time we successfully construct Wasserstein adversarial examples on the ImageNet dataset.
A quick comparison of projected Sinkhorn and our algorithms is shown in \Cref{tb:summary}.
Finally, we show that employing our stronger and faster attacks in adversarial training can significantly improve the robustness of adversarially trained models.
Our implementation is available at \url{https://github.com/watml/fast-wasserstein-adversarial}.

\section{Formulation}
\label{sec:prob}

Wasserstein distance (a.k.a. earth mover's distance) is a metric defined on the space of finite measures with equal total mass \citep{PeyreCuturi19}.
For images, we view them as discrete measures supported on pixel locations. Let $\xv, \zv \in [0,1]^{n}$ be two vectorized images such that $\one^\top\xv = \one^\top \zv$ (equal mass). Their Wasserstein distance is defined as:
\begin{align}
\label{eq:wasserstein_distance}
    \Wc(\xv, \zv) = \min_{\Pi \geq 0} ~ \langle \Pi, C \rangle ~ \st \Pi \one = \xv, \Pi^\top \one = \zv,
\end{align}
where $C \in \Rb^{n \times n}$ is a cost matrix, with $C_{ij}$ representing the cost of transportation from the $i$-th to the $j$-th pixel; and $\Pi \in \Rb^{n \times n}$ is a transportation/coupling matrix, with $\Pi_{ij}$ representing the amount of mass transported from the $i$-th to the $j$-th pixel. Intuitively, Wasserstein distance measures the minimum cost to move mass from $\xv$ to $\zv$. Unlike usual statistical divergences (\eg KL), Wasserstein distance takes the distance between pixels into account hence able to capture the underlying geometry.
It has been widely used in statistics, image processing, graphics, machine learning, \etc, see the excellent monograph \citep{PeyreCuturi19}.

Throughout the paper, w.l.o.g. we assume that all entries in the cost matrix $C$ are nonnegative and $C_{ij} = 0 \Leftrightarrow i = j$.
All common cost matrices satisfy this assumption.

\subsection{PGD with projected Sinkhorn}

Given a deep model that already minimizes some training loss $\EE \ell(X, Y; \thetav)$,
we fix (hence also suppress in notation) the model parameter $\thetav$ and 
aim to generate adversarial examples by \emph{maximizing} the loss $\ell$ subject to some perturbation budget on an input image $\xv$.
Following \textcite{WongSK19}, we use the Wasserstein distance to measure perturbations:
\begin{align}
\label{eq:wasserstein_constrained_optimization_image_space}
    \maxi_{\zv \in [0,1]^n} ~ \ell(\zv, y) ~ \st ~ \Wc(\xv, \zv) \leq \delta = \epsilon \one^\top\xv,
\end{align}
where the perturbation budget $\delta$ is proportional to the total mass in the input image $\xv$ and $\epsilon$ indicates the ``proportion''.
We focus on untargeted attack throughout, where $y$ is the true label of the input image $\xv$.
All techniques in this paper can be easily adapted for targeted attacks as well.

\begin{table*}[t]
\centering
\caption{(Exact) Wasserstein distances $\Wc(\av, \hat{\pv})$ and number of dual iterations in projected Sinkhorn in the first four columns.
$\gamma$ is the entropic regularization constant.
Projected Sinkhorn encountered numerical issues for small $\gamma = 5 \cdot 10^{-5}$.
}
\begin{tabular}{c | c c | c c | c c | c c | c c | c}
\toprule
\multirow{2}{*}{$\gamma$} & \multicolumn{2}{c|}{$10^{-3}$} & \multicolumn{2}{c|}{$2 \cdot 10^{-4}$} & \multicolumn{2}{c|}{$10^{-4}$} &  \multicolumn{2}{c|}{$5 \cdot 10^{-5}$} &  \multicolumn{2}{c|}{ours} & \multirow{2}{*}{ground truth} \\
& $\Wc$ & iter & $\Wc$ & iter & $\Wc$ & iter & $\Wc$ & iter & PGD & FW & \\
\midrule
$\epsilon = 0.5$ & $0.267$ & $ 28$ & $0.402$ & $ 44$ & $0.437$ & $205$ & $-$ & $-$ & $0.500$ & $0.500$ & $0.500$ \\
$\epsilon = 1.0$ & $0.356$ & $ 21$ & $0.498$ & $111$ & $0.555$ & $197$ & $-$ & $-$ & $0.797$ & $0.797$ & $0.797$ \\
\bottomrule
\end{tabular}
\label{tab:sinkhorn_proj_experiment}
\end{table*}

To optimize \eqref{eq:wasserstein_constrained_optimization_image_space}, \textcite{WongSK19} developed an \emph{approximate} projection operator to the Wasserstein ball constraint, called projected Sinkhorn, to enable the projected gradient (PGD) adversarial attack \citep{MadryMSTV18}.
The approximate projection is based on solving an entropic regularized quadratic program (see \Cref{sec:appendix_projected_sinkhorn}).
However, we observed that this approximation is not always accurate in practice.
To test this, we randomly generate two vectors $\av, \bv \in [0,1]^{400}$ with unit total mass.
The initial Wasserstein distance between $\av$ and $\bv$ is $0.797$.
Next, we project $\bv$ using projected Sinkhorn onto Wasserstein balls centered at $\av$ with two different radii $\epsilon = 0.5$ and $\epsilon = 1.0$, respectively. 
We report in \Cref{tab:sinkhorn_proj_experiment} the number of iterations for projected Sinkhorn to output an approximate projection $\hat\pv$, and we compute the (exact) Wasserstein distance between $\av$ and $\hat\pv$ using a linear programming solver for \eqref{eq:wasserstein_distance}.

In the first row of \Cref{tab:sinkhorn_proj_experiment}, we expect $\Wc(\av, \hat\pv) = 0.5$, since  for an exterior point the exact projection should be at the boundary of the Wasserstein ball.
In the second row, we expect $\Wc(\av, \hat\pv) = 0.797$, since an exact projection of an interior point should be itself.
However, in both cases, the actual Wasserstein distances of projected Sinkhorn are always much smaller than the ground truth.
Although $\Wc(\av, \hat\pv)$ gets closer to the ground truth as $\gamma$ (the entropic regularization constant) decreases, non-negligible gaps remain.
Further decreasing $\gamma$ may potentially reduce the approximation error, but (a) overly small $\gamma$ causes numerical issues easily and (b) the number of iterations increases as $\gamma$ decreases.
As we will confirm in \S\ref{sec:exp}, this approximation error in projected Sinkhorn leads to a substantially suboptimal attack.
In comparison, we minimize the quadratic objective in (Euclidean) projection iteratively by PGD with dual projection (\S\ref{sec:dual_proj}) and by FW with dual LMO (\S\ref{sec:frank_wolfe}).
Their outputs are exact in the first three digits after the decimal point, which serves as a simple sanity check of our algorithms in the convex setting.

\subsection{Our Reformulation}
\label{sec:reformulation}

The large approximation error in projected Sinkhorn motivates us to develop more accurate algorithms for stronger attacks. First, we slightly reformulate \eqref{eq:wasserstein_constrained_optimization_image_space} to simplify the constraint. We expand the constraint in \eqref{eq:wasserstein_constrained_optimization_image_space}, and jointly maximize the objective over $\zv$ and $\Pi$:
\begin{align*}
    \maxi_{\zv, \Pi \geq 0} ~~ & \ell(\zv, y) \\
    \sbjto ~ & \Pi \one = \xv, ~ \Pi^\top \one = \zv, ~ \langle \Pi, C \rangle \leq \delta.
\end{align*}
Note that we have dropped the domain constraint $\zv \in [0,1]^n$, which we will revisit in \S\ref{sec:hypercube}.
We plug in the constraint $\Pi^\top \one = \zv$ into the objective to eliminate $\zv$:
\begin{align}
\label{eq:wasserstein_constrained_optimization_coupling}
    \maxi_{\Pi \geq 0} ~ & \ell(\Pi^\top \one, y)
    ~\st ~  \Pi \one = \xv, ~ \langle \Pi, C \rangle \leq \delta,
\end{align}
arriving at a constrained optimization problem \wrt $\Pi$ alone with two linear constraints on it.
Yet, problem \eqref{eq:wasserstein_constrained_optimization_image_space} and \eqref{eq:wasserstein_constrained_optimization_coupling} are clearly equivalent.
Moreover, problem \eqref{eq:wasserstein_constrained_optimization_coupling} has its own interpretation: instead of maximizing the loss in the image space, it maximizes the loss in the transportation space, searching for a feasible transportation plan $\Pi$ with cost no greater than $\delta$, to transport $\xv$ to an adversarial example.
Given a solution $\Pi$ to \eqref{eq:wasserstein_constrained_optimization_coupling}, we generate adversarial examples by summing over the columns of $\Pi$, \ie, $\xv_{adv} = \Pi^\top \one$. We note that $\nabla_{\Pi} \ell = \one (\nabla_{\xv_{adv}} \ell)^\top$, where  $\nabla_{\xv_{adv}} \ell$ can be computed efficiently using backpropagation.

We propose PGD and FW to optimize \eqref{eq:wasserstein_constrained_optimization_coupling}.
While both PGD and FW have been previously used to generate adversarial examples \parencite{MadryMSTV18,ChenZYG20}, they are based on the $\ell_p$ threat model that measures image perturbations under the $\ell_p$ distance.
Instead, for the Wasserstein problem in \eqref{eq:wasserstein_constrained_optimization_coupling}, 
applying PGD and FW requires specialized algorithms for the projection operator and the linear minimization oracle, which are the main goals of \S\ref{sec:dual_proj} and \S\ref{sec:frank_wolfe}.

\section{Projected Gradient with Dual Projection}
\label{sec:dual_proj}

We apply PGD to maximize \eqref{eq:wasserstein_constrained_optimization_coupling} for generating Wasserstein adversarial examples, and arrive at the following update rule on the coupling matrix $\Pi$:
\begin{align*}
    \Pi^{(t + 1)} &= \projto{\Cc}{G}, ~~ \mbox{where} \\
    G &= \Pi^{(t)} + \eta_t \nabla_{\Pi} \ell \big((\Pi^{(t)})^\top \one, y\big)
\end{align*}
and $\projto{\Cc}{\cdot}$ denotes the (Euclidean) projection operator onto the convex set $\Cc$, represented by the constraints in \eqref{eq:wasserstein_constrained_optimization_coupling}. Namely, we take a gradient step and then project it back to the feasible set.
The projection operator $\projto{\Cc}{G}$ is given by the following quadratic program:
\begin{align}
\label{eq:coupling_projection}
    \mini_{\Pi \geq 0} & ~ \tfrac{1}{2} \| \Pi - G \|_\mathrm{F}^2 ~
    \st  ~\Pi \one = \xv, ~ \langle \Pi, C \rangle \leq \delta.
\end{align}
While any quadratic programming solver (\eg interior point method) could be used to solve \eqref{eq:coupling_projection}, they do not scale well in high dimensions.
For high resolution images, \eqref{eq:coupling_projection} could involve millions of variables.
Since this projection is called in each iteration of PGD, it needs to be solved by a highly efficient algorithm.
Below, we exploit the structure of this problem to design a fast specialized projection operator.

\subsection{A First Attempt: Dykstra's Projection}

A simple observation is that the constraint in \eqref{eq:coupling_projection} is precisely the intersection of the following two convex sets:
\begin{align*}
    \Cc_s &= \left\{ \Pi \in \Rb^{n \times n} : \Pi \geq 0, ~ \Pi \one = \xv \right\}, ~ \mbox{and} \\
    \Cc_h &= \left\{ \Pi \in \Rb^{n \times n}: \langle \Pi, C \rangle \leq \delta \right\},
\end{align*}
where each row of $\Cc_s$ is a simplex constraint, requiring Wassserstein adversarial examples to preserve total mass; and $\Cc_h$ is a half space constraint, restricting the perturbation budget of Wasserstein adversarial examples.
It is known that projection to the intersection of convex sets can be computed by Dykstra's algorithm \parencite{BoyleDykstra86, Dykstra83}, provided that the projection to each convex set can be computed easily.
Hence, our first attempt is to apply Dykstra's algorithm to solve \eqref{eq:coupling_projection} (see \Cref{sec:appendix_dykstra} for a description of Dykstra's algorithm and methods for projection to each convex set).
However, the convergence rate of Dykstra's algorithm highly relies on the geometry of these two convex sets $\Cc_s$ and $\Cc_h$ \parencite{DeutschHundal94}.
In fact, we observe that Dykstra's algorithm converges very slowly in some cases (see \Cref{sec:appendix_dykstra}).

\subsection{Dual Projection}

Instead, we develop a dual projection method which converges much faster than Dykstra's algorithm. By the method of Lagrange multipliers, we derive the dual problem: 
\begin{restatable}{proposition}{projectiondualobjective}
\label{prop:projection_dual_objective}
The dual of \eqref{eq:coupling_projection} is
\begin{align}
\label{eq:projection_dual_problem}
    &\maxi_{\lambda \geq 0}~ g(\lambda), ~~~ \mbox{where} \\
\label{eq:projection_dual_objective}
    g(\lambda) = \min_{\Pi \one = \xv, \Pi \geq 0}&~ \tfrac{1}{2} \| \Pi - G \|_{\mathrm{F}}^2 + \lambda \left(\langle \Pi, C \rangle - \delta \right).
\end{align}
In addition, the derivative of $g(\lambda)$ at a point $\lambda = \tilde\lambda$ is
\begin{align}
\label{eq:projection_dual_objective_gradient}
    g^\prime(\tilde\lambda) = \langle \tilde\Pi, C \rangle - \delta, ~~~ \mbox{where} \\
    \tilde\Pi = \argmin_{\Pi \one = \xv, \Pi \geq 0} ~ \| \Pi - G + \tilde \lambda C \|_{\mathrm{F}}^2.
\end{align}
Both $g(\lambda)$ and $g^\prime(\lambda)$ can be evaluated in $O(n^2 \log n)$ time deterministically for any given $\lambda$.
\end{restatable}

The derivation of \Cref{prop:projection_dual_objective} is based on the observation that \eqref{eq:projection_dual_objective} is equivalent to computing a projection operator $\projto{\Cc_s}{G - \lambda C}$.
Since the constraint $\Cc_s$ is independent for each row, it can be further reduced to projecting each row of $G - \lambda C$ to a simplex.
The well-known simplex projection algorithm \parencite[\eg][]{DuchiSSC08} takes $O(n \log n)$ time, thus the projection to $\Cc_s$ takes $O(n^2 \log n)$ time.

Although this dual problem does not have a closed form expression, \Cref{prop:projection_dual_objective} provides a method to evaluate its objective and gradient, which is sufficient to use first order optimization algorithms.
Here, we choose a simple algorithm with linear convergence rate, by exploiting the fact that the dual objective \eqref{eq:projection_dual_objective} is a univariate function.
First, we derive an upper bound of the dual solution.

\begin{restatable}{proposition}{projectiondualupperbound}
\label{prop:projection_dual_upper_bound}
The dual solution $\lambda^\star$ of \eqref{eq:projection_dual_problem} satisfies
\begin{align}
\label{eq:projection_dual_upper_bound}
    0 \leq \lambda^\star \leq \frac{2  \norm{\mathrm{vec}(G)}_\infty + \norm{\xv}_\infty}{\min_{i \neq j} \{C_{ij} \}}.
\end{align}
\end{restatable}

Since $g(\lambda)$ is concave and differentiable, we make the following simple observation: (a) any point $l > 0$ with positive derivative is a lower bound of $\lambda^\star$;
(b) any point $u > 0$ with negative derivative is an upper bound of $\lambda^\star$.
Thus, we start with the lower bound and upper bound in \Cref{prop:projection_dual_upper_bound}, and use bisection method to search for $\lambda^\star$, by iteratively testing the sign of the derivative.
Eventually, the bisection method converges to either a stationary point or the boundary $\lambda = 0$, which are exactly maximizers in both cases.
Since the bisection method halves the gap between lower and upper bounds in each iteration, it converges linearly.
Once we solve the dual, we can recover the primal solution by the following.

\begin{restatable}{proposition}{projectiondualrecover}
\label{prop:projection_dual_recover}
The primal solution $\Pi^\star$ and the dual solution $\lambda^\star$ satisfies
\begin{align*}
    \Pi^\star = \argmin_{\Pi \one = \xv, \Pi \geq 0} ~ \| \Pi - G + \lambda^\star C \|_{\mathrm{F}}^2,
\end{align*}
thus $\Pi^\star$ can be computed in $O(n^2 \log n)$ time given $\lambda^\star$.
\end{restatable}

\begin{algorithm2e}[t]
  \DontPrintSemicolon
  \KwIn{$G, C \in \Rb^{n \times n}$, $\xv \in \Rb^{n}$, $\delta > 0$, $l = 0$, $u > 0$}
  \KwOut{$\tilde\Pi \in \Rb^{n \times n}$}
 \label{line:bisection_search_judge_convergence} \While{not converged}{
    $\tilde{\lambda} = \frac{1}{2}\left(l + u\right)$

    $\tilde\Pi = \argmin_{\Pi \one = \xv, \Pi \geq 0} ~ \| \Pi - G + \tilde\lambda C \|_{\mathrm{F}}^2$
    \label{line:dual_projection_recover}

    \lIf{$\langle \tilde{\Pi}, C \rangle > \delta$}{
      $l = \tilde{\lambda}$
    }
    \lElse{
      $u = \tilde{\lambda}$
    }
  }
  \caption{Dual Projection}
  \label{alg:dual_projection}
\end{algorithm2e}

The full dual projection is presented in \Cref{alg:dual_projection}, where $u$ is initialized as the upper bound \eqref{eq:projection_dual_upper_bound}.
A discussion of the stopping criterion in \Cref{line:bisection_search_judge_convergence} is deferred to \Cref{sec:appendix_bisection_converge}.

Finally, when the loss $\ell$ is convex (concave) in $\xv$, it is also convex (concave) in $\Pi$ after the reformulation in \S\ref{sec:reformulation}.
In this case, projected gradient with dual projection is \emph{guaranteed} to converge to a global optimum. When $\ell$ is nonconvex, projected gradient still converges to a stationary point.

\section{Frank-Wolfe with Dual LMO}
\label{sec:frank_wolfe}

In this section, we apply the Frank-Wolfe algorithm \parencite{FrankWolfe56} to maximize \eqref{eq:wasserstein_constrained_optimization_coupling}. During each iteration, FW first solves the following linear minimization problem:
\begin{align}
\label{eq:lmo}
    \hat\Pi &= \argmin_{\Pi \in \Cc} \left\langle \Pi, H \right\rangle, ~~ \mbox{ where }\\
    H &= - \nabla_{\Pi} \ell \big((\Pi^{(t)})^\top \one, y\big)
\end{align}
and $\Cc$ is the convex set represented by constraints in \eqref{eq:wasserstein_constrained_optimization_coupling}.
Then, we take a convex combination of $\Pi^{(t)}$ and $\hat\Pi$:
\begin{align*}
    \Pi^{(t + 1)} = (1 - \eta_t) \Pi^{(t)} + \eta_t \hat\Pi .
\end{align*}
Step \eqref{eq:lmo} is referred as the linear minimization oracle (LMO) in the literature, and is reduced to solving the following linear program in each iteration:
\begin{align}
\label{eq:frank_wolfe_lmo_primal}
\begin{split}
    \mini_{\Pi \geq 0} ~ & \langle \Pi, H \rangle ~
    \st ~ \Pi \one = \xv, ~ \langle \Pi, C \rangle \leq \delta.
\end{split}
\end{align}
Standard linear programming solvers do not scale well in high dimensions. Instead, we exploit the problem structure again to design a fast, specialized algorithm for \eqref{eq:frank_wolfe_lmo_primal}.

\subsection{A First Attempt: Optimizing the Dual}
Our fist attempt is to extend the idea in \S\ref{sec:dual_proj} to the linear minimization step.
We first derive an equivalent dual problem via the method of Lagrange multipliers:
\begin{restatable}{proposition}{frankdualobjective}
\label{prop:frank_dual_objective}
The dual problem of \eqref{eq:frank_wolfe_lmo_primal} is
\begin{align}
\label{eq:frank_dual_objective}
    \maxi_{\lambda \geq 0} ~ - \lambda \delta + \sum_{i = 1}^{n} x_i \min_{1 \leq j \leq n} \Big(H_{ij} + \lambda C_{ij}\Big).
\end{align}
\end{restatable}
In addition, we provide an upper bound of the maximizer.
\begin{restatable}{proposition}{frankdualupperbound}
\label{prop:frank_dual_upper_bound}
The dual solution $\lambda^\star$ of \eqref{eq:frank_dual_objective} satisfies
\begin{align}
\label{eq:frank_dual_upper_bound}
    0 \leq \lambda^\star \leq \frac{2 \norm{\mathrm{vec}(H)}_\infty}{\min_{i \neq j} \left\{C_{ij}\right\}}.
\end{align}
\end{restatable}

Unlike dual projection, the dual objective \eqref{eq:frank_dual_objective} here is not differentiable.
In fact, it is piecewise linear.
Nevertheless, one can still solve it using derivative-free methods such as bisection method on the supergradient, or golden section search on the objective, both of which converge linearly.

However, after obtaining the dual solution, one cannot recover the primal solution easily.
Consider the following recovery rule by minimizing the Lagrangian:
\begin{align}
\label{eq:lmo_primal_recovery}
    \Pi^\star \in \argmin_{\Pi \geq 0, \Pi \one = \xv} ~ \langle \Pi, H + \lambda^\star C \rangle - \lambda^\star \delta,
\end{align}
where $\Pi^\star$ and $\lambda^\star$ are primal and dual solutions respectively.
There are two issues: (a) there might be infinitely many solutions to \eqref{eq:lmo_primal_recovery} and it is not easy to determine $\Pi^\star$ among them;
(b) even if the solution to \eqref{eq:lmo_primal_recovery} is unique, a slight numerical perturbation could change the minimizer drastically.
In practice, such instability may even result in an infeasible $\Pi$, generating \emph{invalid} adversarial examples outside the Wasserstein perturbation budget.
We direct readers to \Cref{sec:appendix_failure_lmo} for a concrete example and further discussions.

\subsection{Dual LMO via Entropic Regularization}


To address the above issues, we instead solve an entropic regularized version of \eqref{eq:frank_wolfe_lmo_primal} as an approximation:
\begin{align}
\label{eq:frank_wolfe_entropy_primal_problem}
\begin{split}
    & \mini_{\Pi \geq 0} ~~ \langle \Pi, H \rangle + \gamma \sum_{i=1}^{n}\sum_{j=1}^{n} \Pi_{ij} \log \Pi_{ij} \\
    & \sbjto ~ \Pi \one = \xv, ~ \langle \Pi, C \rangle \leq \delta,
\end{split}
\end{align}

where $\gamma$ is a regularization parameter.
The new objective in \eqref{eq:frank_wolfe_entropy_primal_problem} is strongly convex after the entropic regularization.
For any dual variable, the corresponding primal variable minimizing the Lagrangian is always unique, which allows us to recover the primal solution from dual easily.
\begin{restatable}{proposition}{frankwolfeentropydualproblem}
\label{prop:frank_wolfe_entropy_dual_problem}
The dual problem of \eqref{eq:frank_wolfe_entropy_primal_problem} is
\begin{align}
\label{eq:frank_wolfe_entropy_dual_problem}
    \maxi_{\lambda \geq 0} & - \lambda \delta + \gamma \sum_{i : x_i > 0} x_i \log x_i \\
                           & - \gamma \sum_{i = 1}^{n} x_i \log \sum_{j = 1}^{n} \exp\left(- \frac{H_{ij} + \lambda C_{ij}}{\gamma}\right). \nonumber
\end{align}
\end{restatable}
The third term in \eqref{eq:frank_wolfe_entropy_dual_problem} is essentially a softmin operator along each row of $H + \lambda C$, by observing that
\begin{align*}
    \lim_{\gamma \to 0} - \gamma \log \sum_{j = 1}^{n} {\textstyle \exp\left(- \frac{H_{ij} + \lambda C_{ij}}{\gamma}\right)} \!=\! \min_{1 \leq j \leq n}{\textstyle \Big(H_{ij} + \lambda C_{ij}\Big)}.
\end{align*}
Thus, from the point view of dual, \eqref{eq:frank_wolfe_entropy_dual_problem} is a smooth approximation of \eqref{eq:frank_dual_objective}, recovering \eqref{eq:frank_dual_objective} precisely as $\gamma \to 0$. In our implementation, we use the usual log-sum-exp trick to enhance the numerical stability of the softmin operator.

With entropic regularization, we have the following recovery rule for primal solution and upper bound on dual solution.
\begin{restatable}{proposition}{frankwolfeentropydualrecover}
\label{prop:frank_wolfe_entropy_dual_recover}
The primal solution $\Pi^\star$ and the dual solution $\lambda^\star$ satisfy
\begin{align}
\label{eq:frank_wolfe_entropy_dual_recover}
    \Pi_{ij}^\star = x_i \cdot \frac{\exp\left(- \frac{H_{ij} + \lambda^\star C_{ij}}{\gamma}\right)}{\sum_{j = 1}^{n} \exp\left(- \frac{H_{ij} + \lambda^\star C_{ij}}{\gamma}\right)}. 
\end{align}
\end{restatable}

\begin{restatable}{proposition}{frankentropydualupperbound}
\label{prop:frank_entropy_dual_upper_bound}
The dual solution $\lambda^\star$ of \eqref{eq:frank_wolfe_entropy_dual_problem} satisfies
\begin{align}
\label{eq:frank_wolfe_entropy_dual_upper_bound}
    0 \leq \lambda^\star \leq \frac{ \left[2 \norm{\mathrm{vec}(H)}_\infty + \gamma \log \left(\frac{1}{\delta}\xv^\top C \one \right)\right]_+}{\min_{i \neq j}\left\{ C_{ij} \right\}}.
\end{align}
\end{restatable}

Note that the recovery rule \eqref{eq:frank_wolfe_entropy_dual_recover} is essentially applying the softmin activation along each row of $H + \lambda^\star C$.
With the upper bound \eqref{eq:frank_wolfe_entropy_dual_upper_bound}, we can apply the same technique in \S\ref{sec:dual_proj} to solve the dual.
In particular, the bisection method on the dual objective results in an algorithm almost identical to \Cref{alg:dual_projection} with only two exceptions:
(a) we replace the upper bound \eqref{eq:projection_dual_upper_bound} with \eqref{eq:frank_wolfe_entropy_dual_upper_bound} and
(b) we replace \Cref{line:dual_projection_recover} with the primal recover rule \eqref{eq:frank_wolfe_entropy_dual_recover}.

Unlike dual projection, the second order derivative of \eqref{eq:frank_wolfe_entropy_dual_problem} can be computed in a closed form, which enables the usage of second order methods, \eg, Newton's method, for further acceleration.
However, we observed that Newton's method fails to converge in some cases.
The smooth dual \eqref{eq:frank_wolfe_entropy_dual_problem}, as an approximation of \eqref{eq:frank_dual_objective}, still behaves similar to a piecewise linear function after the regularization.
Pure Newton's method might easily overshoot in this case.
Thus, we still choose bisection method due to its stability and relatively fast convergence rate.
In other applications where the convergence speed of the dual is a concern, it is possible to consider second order methods for acceleration.

Although both projected Sinkhorn \parencite{WongSK19} and dual LMO use entropic approximation, we emphasize that there are two key differences.
First, the entropic regularization in dual LMO does not affect the convergence rate.
In contrast, the convergence rate of projected Sinkhorn highly depends on $\gamma$.
In particular, small $\gamma$ often slows down the convergence rate empirically.
Second, unlike the entropic regularized quadratic program used in projected Sinkhorn, entropic regularized linear program like \eqref{eq:frank_wolfe_entropy_primal_problem} has been well studied.
Applications in optimal transport \parencite{Cuturi13} have demonstrated its empirical success; theoretical guarantees on the exponential decay of approximation error have been established \parencite{CominettiMartin94,Weed18}.

For a thorough discussion on the convergence properties of FW on convex or nonconvex functions and beyond, we direct readers to \citep{YuZS17}.

\vspace{-.4em}

\section{Practical Considerations}
In this section, we comment on some practical considerations for implementations of algorithms in \S\ref{sec:dual_proj} and \S\ref{sec:frank_wolfe}.

\subsection{Acceleration via Exploiting Sparsity}
\label{sec:local_transportation}

Both algorithms in \S\ref{sec:dual_proj} and \S\ref{sec:frank_wolfe} require computation on a full transportation matrix $\Pi \in \Rb^{n \times n}$, which is computationally expensive and memory consuming, especially in high dimensions.
For example, for ImageNet where $n = 224 \times 224 \times 3$, the transportation matrix $\Pi$ has billions of variables.
The memory storage for it
(using single precision floating numbers)
exceeds the limit of most GPUs.

To accelerate computation and reduce memory usage, we enforce a local transportation constraint to impose a \emph{structured sparsity} in the transportation matrix \parencite{WongSK19}.
Specifically, we only allow moving pixels in a $k \times k$ neighborhood (see \Cref{fig:local_transportation}).
For images with multiple channels, we only allow transportation within each channel.
Adversarial examples generated with these restrictions are still valid under the original Wasserstein threat models.

With local transportation, each pixel can be only transported to at most $k^2$ possible locations, thus $\Pi$ is a highly sparse matrix with at most $k^2$ nonzero entries in each row.
Such sparsity reduces the per iteration cost of dual projection from $O(n^2 \log n)$ to $O(n k^2 \log k)$, and reduces that of dual LMO from $O(n^2)$ to $O(n k^2)$, both of which are linear \wrt $n$, treating $k$ (typically much smaller than $n$) as a constant.

Operations on sparse matrices in general are not easy to parallelize on GPUs.
Nevertheless, for dual projection and dual LMO, we do have simple customized strategies to support the sparse operations on GPUs by exploring the sparsity pattern in $\Pi$ (see \Cref{sec:sparse_matrix_operation} for a discussion).

\begin{figure}[t]
\centering
    \includegraphics[width=0.8\linewidth]{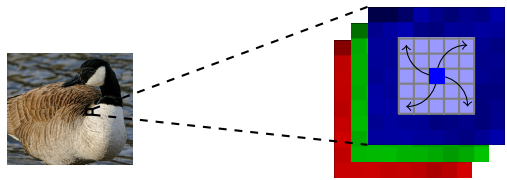}
\caption{An illustration of $5 \times 5$ local transportation.}
\label{fig:local_transportation}
\end{figure}

\subsection{Gradient Normalization}

In both PGD and FW, we normalize the gradient in each iteration by its largest absolute value of entries.
For PGD, gradient normalization is a standard practice, yielding consistent scale of gradient and easing step size tuning.
For FW, normalization in the linear minimization step \eqref{eq:frank_wolfe_lmo_primal} does not change the minimizer, but it does affect the scale of the entropic regularization $\gamma$ in \eqref{eq:frank_wolfe_entropy_primal_problem}.
In this case, normalization keeps the scale of entropic regularization consistent.
In addition, gradient normalization leads to more consistent upper bounds on the dual solutions in \eqref{eq:projection_dual_upper_bound} and \eqref{eq:frank_wolfe_entropy_dual_upper_bound}.

\subsection{Hypercube Constraint in Image Domain}
\label{sec:hypercube}
For image domain adversarial examples, there is an additional hypercube constraint, \eg, $\xv_{adv} \in [0, 1]^n$ if pixels are represented by real numbers in $[0, 1]$.
In practice, we observe that solving problem \eqref{eq:wasserstein_constrained_optimization_coupling} often generates adversarial examples that violate the hypercube constraint, \eg, images with pixel values exceeding $1$.
Although simply clipping pixels can enforce the hypercube constraint, certain amount of pixel mass is lost during clipping, leading to undefined Wasserstein distance.
This is in sharp contrast to $\ell_p$ threat models, where clipping is the typical practice that still retains validness of generated adversarial examples. 

To address this issue, we develop another specialized quadratic programming solver to project a transportation matrix $\Pi$ to the intersection of both constraints.
To the best of our knowledge, the hypercube constraint has not been addressed in previous study of Wasserstein adversarial attacks.
For instance, \textcite{WongSK19} in their implementation simply applied clipping regardless.
This new algorithm, however, is not as efficient as dual projection nor dual LMO.
Thus, we recommend using it as a post-processing procedure on $\Pi^\star$ at the very end
(see \Cref{sec:appendix_post_processing}).
Adversarial attacks tend to be weaker after the post-processing.
However, it ensures that the generated adversarial images satisfy both the  Wasserstein constraint and the hypercube constraint simultaneously hence are genuinely valid.

\section{Experiments}

\label{sec:exp}


\begin{figure}[t]
\centering
\begin{subfigure}{0.48\linewidth}
    \includegraphics[width=\linewidth]{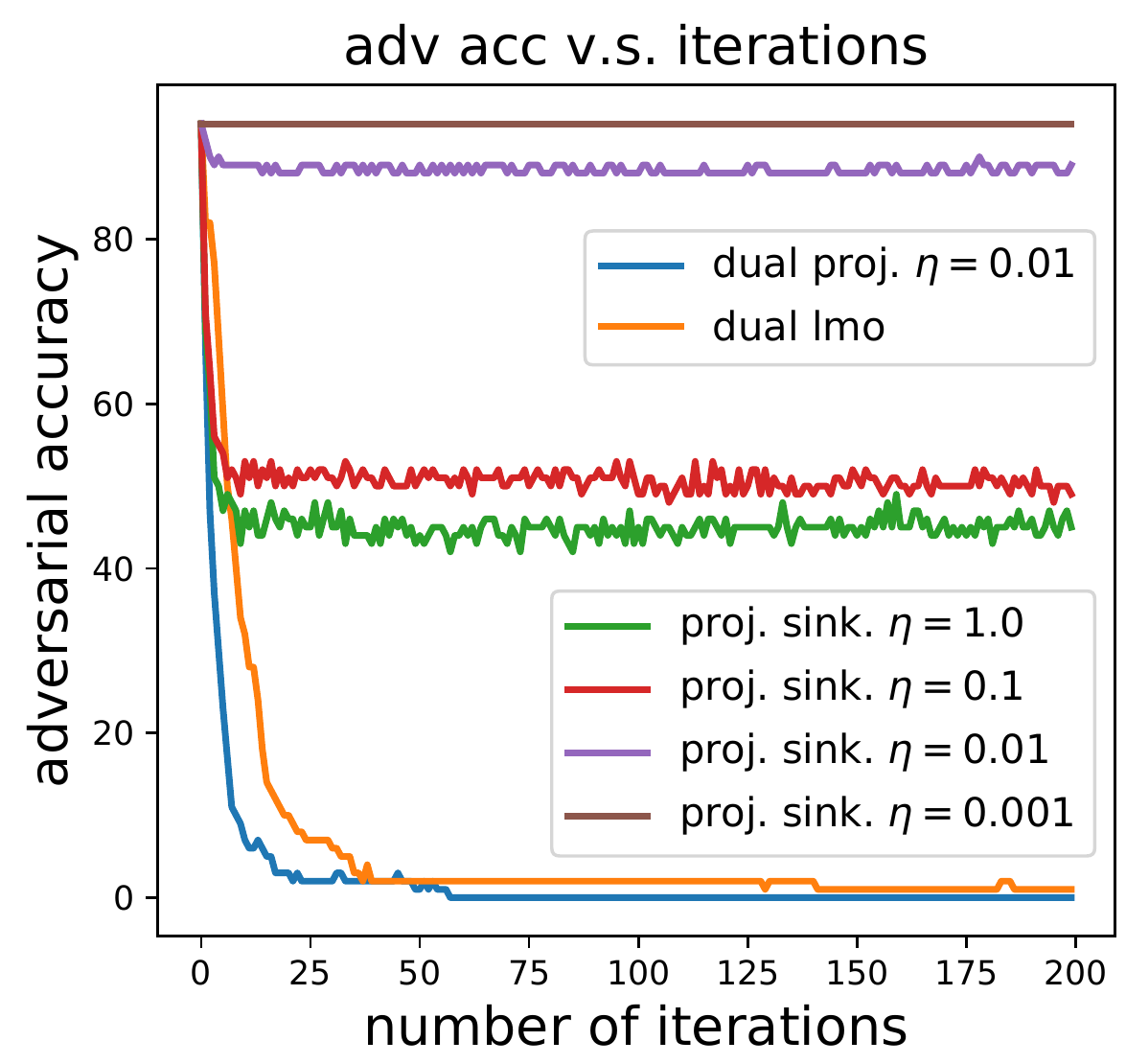}
\caption{CIFAR-10}
\end{subfigure}
\begin{subfigure}{0.48\linewidth}
    \includegraphics[width=\linewidth]{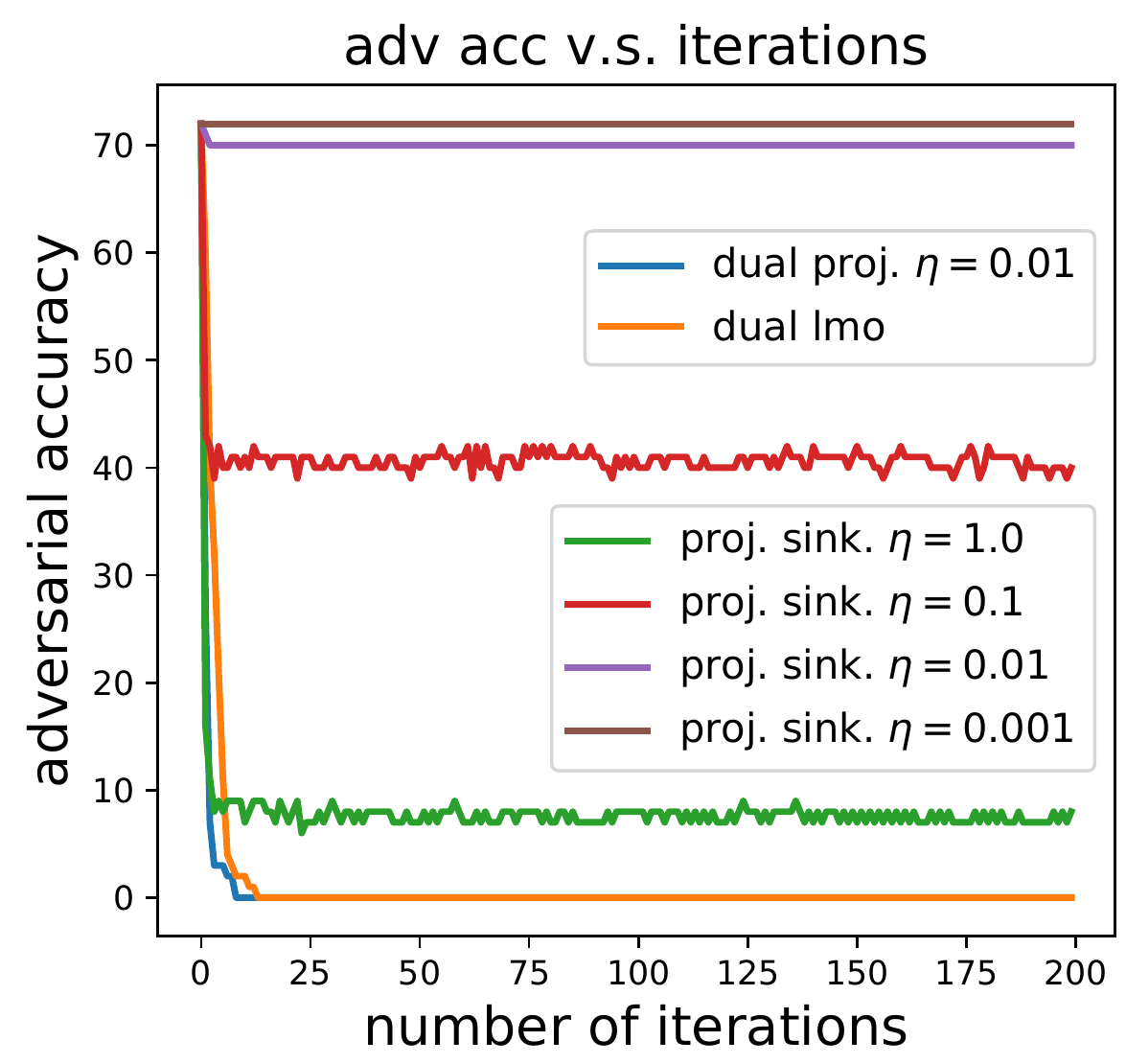}
\caption{ImageNet}
\end{subfigure}
\caption{Adversarial accuracy of models \wrt different iterations of attacks using $\epsilon = 0.005$.
Projected Sinkhorn uses $\gamma = 5 \cdot 10^{-5}$ on CIFAR-10 and $5 \cdot 10^{-6}$ on ImageNet. Dual LMO uses $\gamma = 10^{-3}$ and decay schedule $\frac{2}{t + 1}$.}
\label{fig:cifar_imagenet_converge}
\end{figure}

\textbf{Datasets and models}
We conduct experiments on MNIST \parencite{LeCun98}, CIFAR-10 \parencite{Krizhevsky09} and ImageNet \parencite{DengDSLLL09}.
On MNIST and CIFAR-10, we attack two deep networks used by \citet{WongSK19}.
On ImageNet, we attack a 50-layer residual network \citep{HeZRS16}.
ImageNet experiments are run on the first 100 samples in the validation set.
See model details in \cref{sec:appendix_exp_setting}.

\textbf{Choice of cost}
Throughout the experiment, the cost matrix $C$ is defined as
$C_{(i_1, j_1) (i_2, j_2)} = \sqrt{(i_1 - i_2)^2 + (j_1 - j_2)^2}$,
namely, the Euclidean distance between pixel indices.
We use $5 \times 5$ local transportation plan (see \S\ref{sec:local_transportation}) to accelerate computation of all algorithms.

\textbf{Choice of $\gamma$}
We follow all parameter settings reported by \textcite{WongSK19} for projected Sinkhorn, except that we try a few more $\gamma$. For dual LMO, we use a fixed $\gamma = 10^{-3}$, which we find to work well across different datasets.

\textbf{Optimization parameters}
Stopping criteria of projected Sinkhorn, dual projection and dual LMO, as well as the choice of step sizes of PGD, are deferred to \Cref{sec:appendix_exp_setting}.
FW uses a fixed decay schedule $\eta_t = \frac{2}{t + 1}$. Step sizes of PGD are tuned in $\left\{1, 10^{-1}, 10^{-2}, 10^{-3}\right\}$.
Some experiments for different step sizes are presented in \S\ref{sec:convergence_speed_outer_maximization}.

\subsection{Convergence Speed of Outer Maximization}
\label{sec:convergence_speed_outer_maximization}

Our method depends on a different but equivalent formulation \eqref{eq:wasserstein_constrained_optimization_coupling} that simplifies the constraint. It is reasonble to ask: could the reformulation make the outer maximization in PGD and FW harder? Intuitively, it does not, since the formulation simply embeds a linear transformation before the input layer (summation over columns of $\Pi$). To verify this, we plot adversarial accuracy of models \wrt number of iterations for different attack algorithms in \Cref{fig:cifar_imagenet_converge} to compare their convergence rate of outer maximization.
More thorough results are deferred to \Cref{sec:appendix_loss_acc_vs_iter}.
\footnote{The numbers in  \Cref{fig:cifar_imagenet_converge} are meant for comparison of convergence speed of Wasserstein constrained optimization problem. Thus they are shown without the post-processing algorithm discussed in \S\ref{sec:hypercube} and may differ slightly from those in \Cref{tb:acc_iter_time}.}

We observe that FW with the default decay schedule converges very fast, especially at the initial stage.
Meanwhile, PGD with dual projection converges slightly faster than FW with carefully tuned step sizes.
In contrast, PGD with projected Sinkhorn barely decreases the accuracy when using small step sizes.
The output of projected Sinkhorn is only a feasible point in the Wasserstein ball, rather than an accurate projection, due to the crude approximation.
If using small step sizes, projected Sinkhorn brings the iterates of PGD closer to the center of Wasserstein ball in every iteration.
Thus, the iterates always stay around the center of Wasserstein ball during the optimization, hence cannot decrease the accuracy.
To make progress in optimization, it is required to use aggressively large step sizes (\eg $\eta = 1.0$ and $\eta = 0.1$).

\subsection{Attack Strength and Dual Convergence Speed}

\newcommand{\midsepremove}{\aboverulesep = 0.36mm \belowrulesep = 0.6mm}

\newcommand{\midsepdefault}{\aboverulesep = 0.605mm \belowrulesep = 0.984mm}

\midsepremove

\begin{table*}[t]
\centering
\caption{
Comparison of adversarial accuracy and average number of dual iterations. $\gamma = \frac{1}{1000}$ on MNIST and $\gamma = \frac{1}{3000}$ on CIFAR-10 are the parameters used by \textcite{WongSK19}.
``$-$'' indicates numerical issues during computation.
}
\label{tb:acc_iter_time}

\begin{tabular}{c l r r r r r r r r r r}
\toprule
& \multicolumn{1}{c}{\multirow{2}{*}{method}} & \multicolumn{2}{c}{$\epsilon = 0.1$} & \multicolumn{2}{c}{$\epsilon = 0.2$} & \multicolumn{2}{c}{$\epsilon = 0.3$} & \multicolumn{2}{c}{$\epsilon = 0.4$} & \multicolumn{2}{c}{$\epsilon = 0.5$} \\
& & \multicolumn{1}{c}{acc} & \multicolumn{1}{c}{iter} & \multicolumn{1}{c}{acc} & \multicolumn{1}{c}{iter} & \multicolumn{1}{c}{acc} & \multicolumn{1}{c}{iter} & \multicolumn{1}{c}{acc} & \multicolumn{1}{c}{iter} & \multicolumn{1}{c}{acc} & \multicolumn{1}{c}{iter} \\
\cmidrule{2-12}
\multirow{7}{*}{\rotatebox[origin=c]{90}{MNIST}}
& PGD + Proj. Sink. ($\gamma=1 /  1000$)  & $96.5$ & $ 92$ &  $91.2$ & $ 88$ &  $78.0$ & $ 85$ &  $59.1$ & $ 82$ &  $40.1$ & $ 80$ \\
& PGD + Proj. Sink. ($\gamma=1 /  1500$)  & $95.2$ & $110$ &  $82.3$ & $116$ &  $58.2$ & $112$ &  $-$    & $-$   &  $-$    & $-$   \\
& PGD + Proj. Sink. ($\gamma=1 /  2000$)  & $-$    & $-$     & $-$    & $-$     & $-$    & $-$     & $-$    & $-$     & $-$    & $-$     \\
& PGD + Dual Proj.                        & $63.4$ & $ 15$ &  $13.3$ & $ 15$ &  $ 1.4$ & $ 15$ &  $ 0.1$ & $ 15$ &  $ 0.0$ & $ 15$ \\
& FW ~ + Dual LMO ($\gamma=10^{-3}$)      & $67.5$ & $ 15$ &  $16.9$ & $ 15$ &  $ 2.2$ & $ 15$ &  $ 0.4$ & $ 15$ &  $ 0.1$ & $ 15$ \\
\cmidrule{2-12}
& PGD + Dual Proj. ~(w/o post-processing)  & $42.6$ & $ 15$ &  $ 4.2$ & $ 15$ &  $ 0.4$ & $ 15$ &  $ 0.0$ & $ 15$ &  $ 0.0$ & $ 15$ \\
& FW ~ + Dual LMO (w/o post-processing)     & $48.3$ & $ 15$ &  $ 6.7$ & $ 15$ &  $ 1.0$ & $ 15$ &  $ 0.3$ & $ 15$ &  $ 0.1$ & $ 15$ \\
\midrule

& \multicolumn{1}{c}{\multirow{2}{*}{method}} & \multicolumn{2}{c}{$\epsilon = 0.001$} & \multicolumn{2}{c}{$\epsilon = 0.002$} & \multicolumn{2}{c}{$\epsilon = 0.003$} & \multicolumn{2}{c}{$\epsilon = 0.004$} & \multicolumn{2}{c}{$\epsilon = 0.005$} \\
& & \multicolumn{1}{c}{acc} & \multicolumn{1}{c}{iter} & \multicolumn{1}{c}{acc} & \multicolumn{1}{c}{iter} & \multicolumn{1}{c}{acc} & \multicolumn{1}{c}{iter} & \multicolumn{1}{c}{acc} & \multicolumn{1}{c}{iter} & \multicolumn{1}{c}{acc} & \multicolumn{1}{c}{iter} \\
\cmidrule{2-12}
\multirow{7}{*}{\rotatebox[origin=c]{90}{CIFAR-10}}
& PGD + Proj. Sink. ($\gamma=1 /  3000$)  & $93.0$ & $ 33$ &  $91.3$ & $ 33$ &  $89.5$ & $ 33$ &  $87.6$ & $ 33$ &  $85.7$ & $ 33$ \\
& PGD + Proj. Sink. ($\gamma=1 / 10000$)  & $89.9$ & $ 79$ &  $84.5$ & $ 79$ &  $78.3$ & $ 79$ &  $71.9$ & $ 79$ &  $65.6$ & $ 79$ \\
& PGD + Proj. Sink. ($\gamma=1 / 20000$)  & $-$    & $-$   & $-$     & $-$   & $-$     & $-$   & $-$     & $-$   & $-$     & $-$   \\
& PGD + Dual Proj.                        & $30.3$ & $ 15$ &  $10.5$ & $ 15$ &  $ 5.6$ & $ 15$ &  $ 4.0$ & $ 15$ &  $ 3.4$ & $ 15$ \\
& FW ~ + Dual LMO  ($\gamma=10^{-3}$)     & $33.5$ & $ 15$ &  $13.6$ & $ 15$ &  $ 7.2$ & $ 15$ &  $ 4.7$ & $ 15$ &  $ 3.7$ & $ 15$ \\
\cmidrule{2-12}
& PGD + Dual Proj. ~(w/o post-processing)  & $25.9$ & $ 15$ &  $ 6.0$ & $ 15$ &  $ 1.7$ & $ 15$ &  $ 0.5$ & $ 15$ &  $ 0.2$ & $ 15$ \\
& FW ~ + Dual LMO (w/o post-processing)     & $29.6$ & $ 15$ &  $ 9.1$ & $ 15$ &  $ 3.2$ & $ 15$ &  $ 1.1$ & $ 15$ &  $ 0.6$ & $ 15$ \\
\midrule

& \multicolumn{1}{c}{\multirow{2}{*}{method}} & \multicolumn{2}{c}{$\epsilon = 0.001$} & \multicolumn{2}{c}{$\epsilon = 0.002$} & \multicolumn{2}{c}{$\epsilon = 0.003$} & \multicolumn{2}{c}{$\epsilon = 0.004$} & \multicolumn{2}{c}{$\epsilon = 0.005$} \\
& & \multicolumn{1}{c}{acc} & \multicolumn{1}{c}{iter} & \multicolumn{1}{c}{acc} & \multicolumn{1}{c}{iter} & \multicolumn{1}{c}{acc} & \multicolumn{1}{c}{iter} & \multicolumn{1}{c}{acc} & \multicolumn{1}{c}{iter} & \multicolumn{1}{c}{acc} & \multicolumn{1}{c}{iter} \\
\cmidrule{2-12}
\multirow{7}{*}{\rotatebox[origin=c]{90}{ImageNet}}
& PGD + Proj. Sink. ($\gamma=1 / 100000$)  & $68.0$ & $ 42$ &  $61.2$ & $ 43$ &  $59.2$ & $ 43$ &  $55.8$ & $ 43$ &  $52.4$ & $ 43$ \\
& PGD + Proj. Sink. ($\gamma=1 / 200000$)  & $61.2$ & $ 72$ &  $56.5$ & $ 72$ &  $48.3$ & $ 72$ &  $42.9$ & $ 71$ &  $38.1$ & $ 71$ \\
& PGD + Proj. Sink. ($\gamma=1/1000000$)  & $-$    & $-$     & $-$    & $-$     & $-$    & $-$     & $-$    & $-$     & $-$    & $-$     \\
& PGD + Dual Proj.                        & $ 9.0$ & $ 15$ &  $ 9.0$ & $ 15$ &  $ 8.0$ & $ 15$ &  $ 8.0$ & $ 15$ &  $ 7.0$ & $ 15$ \\
& FW ~ + Dual LMO                         & $10.0$ & $ 15$ &  $ 9.0$ & $ 15$ &  $ 8.0$ & $ 15$ &  $ 8.0$ & $ 15$ &  $ 7.0$ & $ 15$ \\
\cmidrule{2-12}
& PGD + Dual Proj. ~(w/o post-processing) & $ 0.0$ & $ 15$ &  $ 0.0$ & $ 15$ &  $ 0.0$ & $ 15$ &  $ 0.0$ & $ 15$ &  $ 0.0$ & $ 15$ \\
& FW ~ + Dual LMO (w/o post-processing)   & $ 1.0$ & $ 15$ &  $ 0.0$ & $ 15$ &  $ 0.0$ & $ 15$ &  $ 0.0$ & $ 15$ &  $ 0.0$ & $ 15$ \\
\bottomrule
\end{tabular}
\end{table*}

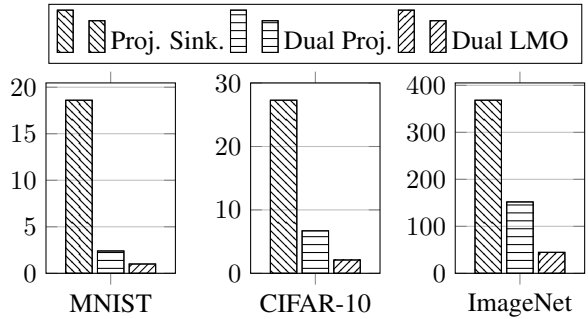
\begin{figure}
\centering
\begin{tikzpicture}
\begin{axis}[
    name=ax1,
    width=0.4\linewidth,
    height=0.18\textheight,
    ymin=0,
    enlarge x limits=0.18,
    ybar,
    xtick=data,
    symbolic x coords={MNIST},
    grid=major,
    xmajorgrids=false,
    ]
\addplot[postaction={pattern=north west lines}] coordinates {(MNIST,18.6)};
\addplot[postaction={pattern=horizontal lines}] coordinates {(MNIST,2.4)};
\addplot[postaction={pattern=north east lines}] coordinates {(MNIST,1.0)};
\end{axis}

\begin{axis}[
    name=ax2,
    at={(ax1.south east)},
    xshift=1cm,
    width=0.4\linewidth,
    height=0.18\textheight,
    legend style={
        at={(0.5,1.1)},
        anchor=south,
        legend columns=-1
    },
    legend image post style={scale=2.0},
    ymin=0,
    enlarge x limits=0.18,
    ybar,
    xtick=data,
    symbolic x coords={CIFAR-10},
    grid=major,
    xmajorgrids=false,
    ]
\addplot[postaction={pattern=north west lines}] coordinates {(CIFAR-10, 27.3)};
\addplot[postaction={pattern=horizontal lines}] coordinates {(CIFAR-10, 6.7)};
\addplot[postaction={pattern=north east lines}] coordinates {(CIFAR-10, 2.1)};
\legend{Proj. Sink., Dual Proj., Dual LMO}
\end{axis}

\begin{axis}[
    name=ax3,
    at={(ax2.south east)},
    xshift=1cm,
    width=0.4\linewidth,
    height=0.18\textheight,
    ymin=0,
    enlarge x limits=0.18,
    ybar,
    xtick=data,
    symbolic x coords={ImageNet},
    grid=major,
    xmajorgrids=false,
    ]
\addplot[postaction={pattern=north west lines}] coordinates {(ImageNet, 368.3)};
\addplot[postaction={pattern=horizontal lines}] coordinates {(ImageNet, 151.9)};
\addplot[postaction={pattern=north east lines}] coordinates {(ImageNet, 44.6)};
\end{axis}
\end{tikzpicture}
\caption{Per iteration running time (in milliseconds) of different algorithms measured on a single P100 GPU.}
\label{fig:running_time_comparison}
\end{figure}

In \Cref{tb:acc_iter_time}, we compare
(a) strength of different attacks by adversarial accuracy, \ie model accuracy under attacks and
(b) the running speed by the average number of dual iterations.
We observe that PGD with dual projection attack and FW with dual LMO attack are generally stronger than PGD with projected Sinkhorn, since the latter is only an \emph{approximate} projection hence it does not solve \eqref{eq:wasserstein_constrained_optimization_image_space} adequately.
As $\gamma$ decreases, PGD with projected Sinkhorn gradually becomes stronger due to better approximation, but at a cost of an increasing number of iterations to converge.
However, projected Sinkhorn is still weaker than PGD with dual projection and FW with dual LMO, even after tuning $\gamma$. Unfortunately, further decreasing $\gamma$ runs into numerical overflow.
We notice that PGD with dual projection is often slightly stronger than FW with dual LMO for two possible reasons: the projection step is solved exactly without any approximation error; we use the default decay schedule in FW.
Tuning the decay schedule for specific problems might improve the attack strength and convergence speed of FW. 


For completeness, we also report the results of dual projection and dual LMO without post-processing in \Cref{tb:acc_iter_time}.
After post-processing (see \S\ref{sec:hypercube}), the adversarial accuracy is increased, sometimes by a lot.
This is especially the case on MNIST (\eg, $\epsilon=0.1$)  where there are many white pixels, thus it is very easy to violate the hypercube constraint.
Note that PGD with projected Sinkhorn might be even weaker than what is indicated by the statistics in \Cref{tb:acc_iter_time}, if we post-process its adversarial examples appropriately so that they are genuinely valid.
However, we do not have an efficient algorithm for post-processing projected Sinkhorn, thus we simply let it ignore the hypercube constraint.
Even so, our attacks are still much stronger than it.

Thanks to the bisection method and the tight upper bounds \eqref{eq:projection_dual_upper_bound} and \eqref{eq:frank_wolfe_entropy_dual_upper_bound}, dual projection and dual LMO converge very fast to high precision solutions.
In practice, they often terminate exactly in $15$ iterations due to the consistent scales of the upper bounds (see \Cref{sec:stop_criterion} for a discussion).
On the other hand, projected Sinkhorn typically requires more dual iterations.
Besides convergence speed, we also compare the real running time of a single dual iteration of all three methods in \Cref{fig:running_time_comparison}.
On MNIST and CIFAR-10, dual projection is $5 \sim 7$ times faster than projected Sinkhorn; while on ImageNet, dual projection is roughly twice faster than projected Sinkhorn.
Meanwhile, dual LMO is $2 \sim 3$ times faster than dual projection due to the absence of the extra logarithm factor.
See
\Cref{sec:appendix_running_time_projected_sinkhorn}
for more details.

\begin{figure}[t]
\centering
    \includegraphics[width=\linewidth]{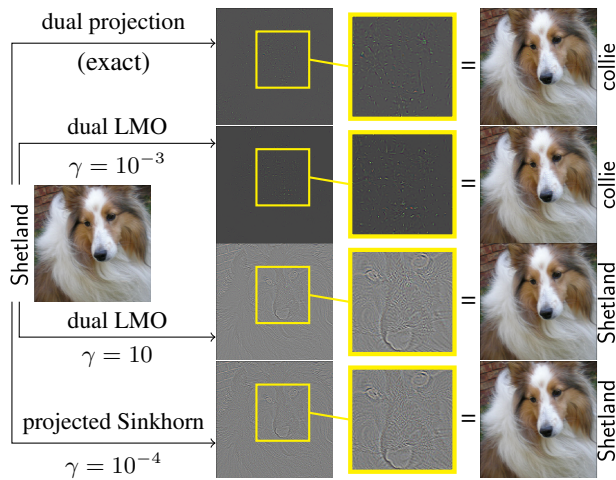}
\caption{Wasserstein adversarial examples ($\epsilon=0.005$) generated by different algorithms on ImageNet.
Perturbations are normalized to $[0, 1]$ for visualization.
Dog faces can be observed after zooming in.
}
\label{fig:imagnet_shetland}
\end{figure}

\subsection{Entropic Regularization Reflects Shapes in Images}

\begin{figure}[t]
\centering
\begin{tikzpicture}
\node[inner sep=0] (adv_imgs) at (0, 0) {\includegraphics[width=\linewidth]{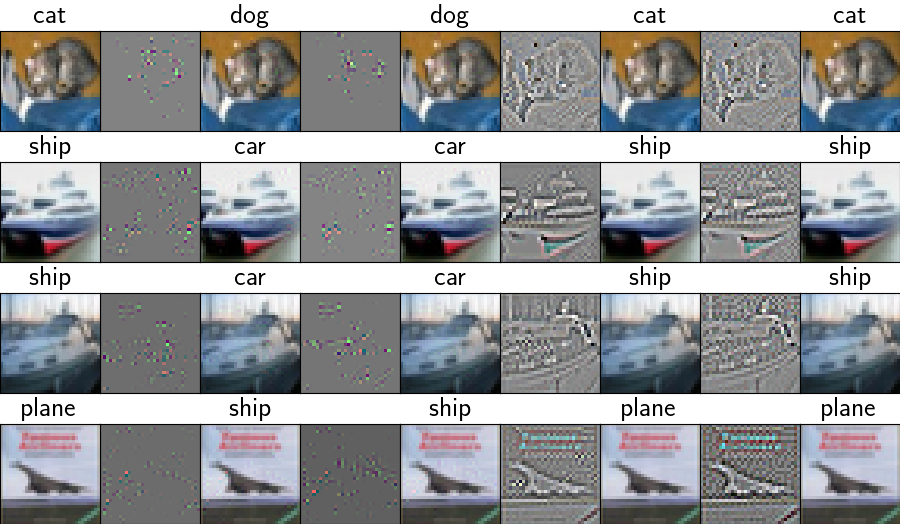}};

\getcorners{adv_imgs}

\drawtext{1}{0.888}{\scriptsize clean}{}
\drawtext{0.888}{0.666}{\scriptsize dual proj.}{}
\drawtext{0.666}{0.444}{\scriptsize dual LMO}{\scriptsize ($\gamma = 10^{-3}$)}
\drawtext{0.444}{0.222}{\scriptsize dual LMO}{\scriptsize ($\gamma = 10$)}
\drawtext{0.222}{0.0}{\scriptsize proj. Sink.}{\scriptsize ($\gamma = 1/3000$)}

\end{tikzpicture}
\caption{Wasserstein adversarial examples ($\epsilon = 0.002$) generated by different algorithms on CIFAR-10. Perturbations reflect shapes in images only when using large entropic regularization (the 6th and 8th columns).
}
\vspace{-.08em}
\label{fig:cifar_adv}
\end{figure}

\textcite{WongSK19} noted that Wasserstein perturbations reflect shapes in original images.
Instead, we argue that it is the large entropic regularization that causes the phenomenon.
We visualize adversarial examples and perturbations generated by different attacks in \Cref{fig:imagnet_shetland} and \Cref{fig:cifar_adv}.
Perturbations generated by PGD with dual projection and FW with dual LMO using small $\gamma$ tend to be very sparse, \ie, only moving a few pixels in the images.
In comparison, we gradually increase the entropic regularization in dual LMO and eventually are able to reproduce the shape reflection phenomenon observed by \textcite{WongSK19}.
The connection between entropic regularization and shape reflection phenomenon is deferred to \Cref{sec:appendix_entropy_relection_analysis}.
The fact that projected Sinkhorn generates adversarial perturbations reflecting shapes in clean images could be another evidence that the entropic regularization is too large.
Notice that large entropic regularization causes large approximation error, thus potentially requires larger $\epsilon$ in order to successfully generate adversarial examples.

\subsection{Improved Adversarial Training}

Since we have developed stronger and faster attacks, it is natural to apply them in adversarial training \parencite{MadryMSTV18} to improve the robustness of adversarially trained models.
Indeed, models adversarially trained by our stronger attack have much higher adversarial accuracy compared with models trained by projected Sinkhorn.
We apply FW with dual LMO in adversarial training due to its fast convergence speed and fast per iteration running speed.
On MNIST, we produce a robust model with $59.7\%$ accuracy against all three attacks with perturbation budget $\epsilon=0.5$, compared with $0.6\%$ using projected Sinkhorn.
On CIFAR-10, we produce a robust model with $56.8\%$ accuracy against all three attacks with perturbation budget $\epsilon=0.005$, compared with $41.2\%$ using projected Sinkhorn.
See \Cref{sec:adversarial_training} for a thorough evaluation of adversarially trained models.


\section{Conclusion}


We have demonstrated that the previous Wasserstein adversarial attack based on \emph{approximate} projection is suboptimal due to inaccurate projection.
To generate stronger Wasserstein adversarial attacks, we introduce two \emph{faster} and more \emph{accurate} algorithms for Wasserstein constrained optimization.
Each algorithm has its own advantage thus they  complement each other nicely:
PGD with dual projection employs exact projection and generates the strongest attack.
On the other hand, with minimal entropic smoothing, FW with dual LMO is extremely fast in terms of both outer maximization and linear minimization step without much tuning of hyperparameters.
Extensive experiments confirm the effectiveness of our algorithms in two ways: (a) properly evaluating Wasserstein adversarial robustness and (b) improving robustness through adversarial training.
Finally, our algorithms impose minimal assumptions on the cost matrix in Wasserstein distance, thus they can be directly applied to other applications involving Wasserstein constrained optimization problems on discrete domains.

\clearpage

\section*{Acknowledgement}
We thank the reviewers for their constructive comments. 
Resources used in preparing this research were provided, in part, by the Province of Ontario, the Government of Canada through CIFAR, and companies sponsoring the Vector Institute. We gratefully acknowledge funding support from NSERC, the Canada CIFAR AI Chairs Program, and Waterloo-Huawei Joint Innovation Lab. We thank NVIDIA Corporation (the data science grant) for donating two Titan V GPUs that enabled in part the computation in this work.

\bibliography{ref}
\bibliographystyle{icml2020}

\clearpage

\appendix

\clearpage

\onecolumn


\section{Projected Sinkhorn}
\label{sec:appendix_projected_sinkhorn}


Here, we give a brief description of the approximate projection method proposed by \citet{WongSK19}.
The projection of a (normalized) vector $\wv$ to the Wasserstein ball centered at (normalized) $\xv$ of radius of $\epsilon = \delta$ can be written as:
\begin{align*}
    & \mini_{\zv, \Pi \geq 0} ~~  \frac{1}{2} \norm{\wv - \zv}_2^2 \\
    & \sbjto ~ \Pi \one = \xv, ~ \Pi^\top \one = \zv, \langle \Pi, C \rangle \leq \epsilon.
\end{align*}
The above objective is not strongly convex in $\Pi$, but can be made strongly convex by adding an entropic regularization:
\begin{align}
\label{eq:projected_sinkhorn_entropy_primal}
\begin{split}
    & \mini_{\zv, \Pi \geq 0} ~~ \frac{1}{2} \norm{\wv - \zv}_2^2 + \gamma \sum_{i=1}^{n}\sum_{j=1}^{n} \Pi_{ij}\log \Pi_{ij} \\
    & \sbjto ~ \Pi \one = \xv, ~ \Pi^\top \one = \zv, \langle \Pi, C \rangle \leq \epsilon.
\end{split}
\end{align}
The parameter $\gamma > 0$ is the entropic regularization constant.
Projected Sinkhorn solves \eqref{eq:projected_sinkhorn_entropy_primal} through block-coordinate maximization on the dual problem of \eqref{eq:projected_sinkhorn_entropy_primal}.

\subsection{Analysis of Approximation Error in Projected Sinkhorn}

To ensure small approximation error in \eqref{eq:projected_sinkhorn_entropy_primal}, the scale of entropic regularization term should be at least much smaller than the quadratic term:
\begin{align*}
    \gamma \sum_{i=1}^{n}\sum_{j=1}^{n} \Pi_{ij}\log \Pi_{ij} \ll \frac{1}{2} \norm{\wv - \zv}_2^2.
\end{align*}
Otherwise, the objective \eqref{eq:projected_sinkhorn_entropy_primal} is dominated by the entropic regularization.
However, in practice, it is not always guaranteed, especially when $\wv$ is an interior point of the constraint in \eqref{eq:projected_sinkhorn_entropy_primal}.

Consider an simple example where $\wv = \xv = (\frac1n, \frac1n, \cdots \frac1n)^\top$.
In that case, the quadratic term $\frac12 \| \xv - \zv \|_2^2$ should be as small as zero, since we can let $\zv = \xv$.
However, if $\zv = \wv = \xv$, then $\Pi$ could be a diagonal matrix $\diag{\frac1n, \frac1n, \cdots \frac1n}$ (or more generally, $\frac1n P$, where $P$ is a permutation matrix).
Thus, the entropic term becomes
\begin{align}
\label{eq:maximum_of_negative_entropy}
    \sum_{i = 1}^{n} \sum_{j = 1}^{n} \Pi_{ij} \log \Pi_{ij} = - \log n,
\end{align}
reaching its \emph{maximum}.
The entropic regularization somewhat conflicts with the the quadratic term. 
Notice that the scale of \eqref{eq:maximum_of_negative_entropy} is much larger than $\frac12 \| \wv - \xv \|_2^2$ (which is supposed to be very close to zero), especially when the dimension $n$ is large.
Thus, the objective \eqref{eq:projected_sinkhorn_entropy_primal} may be dominated by the entropic regularization and solving the projection step accurately requires very small $\gamma$.

We make two additional remarks.
First, notice that the scale of \eqref{eq:maximum_of_negative_entropy} increases as $n$ grows, which requires smaller $\gamma$ to balance the quadratic term and entropic regularization.
This gives the intuition that projected Sinkhorn needs smaller $\gamma$ in higher dimensional spaces, which is observed in experiments.

Second, the key aspect of the above argument is that $\wv$ is relatively close to $\xv$, \eg $\Wc(\wv, \xv) \leq \epsilon$, such that the quadratic term is so small hence dominated by the entropic regularization.
In the case where $\wv$ is very far away from $\xv$, this argument does not hold anymore.
We believe this explains why large step sizes strengthen the attack when using PGD with projected Sinkhorn in experiments.
However, PGD with large step sizes tend to be unstable and may not converge to a good solution.

\subsection{Toy Experiment in \Cref{tab:sinkhorn_proj_experiment}}

Entries of $\av$ and $\bv$ are sampled from a uniform distribution in $[0, 1]^{400}$ independently. After sampling, both vectors are normalized to ensure that the pixel mass summations are exactly $1$. We reshape $\av$ and $\bv$ to $\Rb^{20 \times 20}$ and view them as images in order to use the procedure of \textcite{WongSK19}. The cost matrix is induced by Euclidean norm between pixel indices with $5 \times 5$ local transportation plan.

We use projected gradient descent and Frank-Wolfe to compute the projection by solving the following ``reparametrized'' projection problem \wrt the coupling matrix:
\begin{align*}
    & \mini_{\Pi} ~~ \frac12 \| \Pi^\top \one - \bv \|_2^2 \\
    & \sbjto ~ \Pi \geq 0, \Pi \one = \av, \langle \Pi, C \rangle \leq \epsilon.
\end{align*}
The problem is equivalent to the Euclidean projection in the image space and is convex in $\Pi$.
For PGD, we use step size $0.05$.
For Frank-Wolfe, we use $\gamma=10^{-3}$ and the default decay schedule $\frac{2}{t + 1}$.
We let both algorithms run for sufficiently many iterations in order to converge to high precision solutions.

\section{Recommended Stopping Criterion for Bisection Search}
\label{sec:appendix_bisection_converge}


When the derivative of the dual objective approaches zero, \ie, $\langle \Pi, C \rangle - \delta \approx 0$, the comparison between $\langle \Pi, C \rangle - \delta$ and $0$ is getting numerically unstable.
Thus, we recommend stopping the bisection method when either the derivative is close to zero, or the gap between the lower bound $l$ and the upper bound $u$ is relatively small.

We recommend using an upper bound $u$ to recover the coupling $\Pi$.
Since an upper bound $u$ always has a negative derivative, thus the transportation cost constraint $\langle \Pi, C \rangle \leq \delta$ is always satisfied.

We highlight that the bisection method converges very fast in practice, since it shrinks the interval by a factor of $2$ in every iteration.
Thus, it determines the next $3$ digits of $\lambda^\star$ after the decimal point after every $10$ iterations.

For a concrete stopping criterion in our experiment, please refer to \Cref{sec:stop_criterion}.

\clearpage

\section{Dykstra's Projection}
\label{sec:appendix_dykstra}


\begin{algorithm2e}[th]
  \KwIn{$G \in \Rb^{n \times n}$, two convex sets $\Cc_s$ and $\Cc_h$}
  \KwOut{The projection of $G$ to $\Cc_s \cap \Cc_h$}
  $\Pi_h^{(0)} = G$

  $I_s^{(0)} = I_h^{(0)} = O$

  \For{$t=0, 1, \ldots, \mathsf{maxiter}$}{
    $\Pi_s^{(t + 1)} = \projto{\Cc_s}{\Pi_h^{(t)} - I_s^{(t)}}$

    $I_s^{(t + 1)} = \Pi_s^{(t + 1)} - \Pi_h^{(t)} + I_s^{(t)}$

    $\Pi_h^{(t + 1)} = \projto{\Cc_h}{\Pi_s^{(t + 1)} - I_h^{(t)}}$

    $I_h^{(t + 1)} = \Pi_h^{(t + 1)} - \Pi_s^{(t+ 1)} + I_h^{(t)}$
  }
  \Return $\Pi_s^{(t + 1)}$
  \caption{Dykstra's Projection Algorithm}
  \label{alg:dykstra}
\end{algorithm2e}

Consider the projection of $G \in \Rb^{n \times n}$ to the intersection of $\Cc_s$ and $\Cc_h$, where
    $\Cc_s = \left\{ \Pi \in \Rb^{n \times n} : \Pi \geq 0,~\Pi \one = \xv \right\}$
and
    $\Cc_h = \left\{ \Pi \in \Rb^{n \times n}: \langle \Pi, C \rangle \leq \delta \right\}$.
Dykstra's algorithm, applying to these two convex sets, is presented in \Cref{alg:dykstra}. Intuitively, Dykstra's algorithm projects $G$ alternatively to $\Cc_s$ and $\Cc_h$ in each iteration. Notice that before projecting to $\Cc_s$ (or $\Cc_h$), the increment of the last iteration $I_s^{(t)}$ (or $I_h^{(t)}$) is subtracted from $\Pi_h^{(t)}$ (or $\Pi_s^{(t + 1)}$). These increments play a crucial role in the convergence of Dykstra's algorithm. It has been shown that both $\Pi_s^{(t)}$ and $\Pi_h^{(t)}$ converge to the projection of $G$ onto $\Cc_s \cap \Cc_h$ \parencite{Dykstra83,BoyleDykstra86}.

In order to implement Dykstra's algorithm, we need two subroutines to compute the projection onto $\Cc_s$ and $\Cc_h$ respectively.
The projection onto $\Cc_h$ admits a closed form expression:
\begin{align*}
    \projto{\Cc_h}{\Pi} = \Pi - \frac{\max\{\langle \Pi, C \rangle - \delta, 0\}}{\| C \|_{\mathrm{F}}^2} C.
\end{align*}
The projection onto $\Cc_s$ has an algorithm running in $O(n^2 \log n)$ time, by projecting each row of $G$ to a simplex.
For the simplex projection algorithm, we direct readers to \parencite{DuchiSSC08}.

\subsection{Toy Experiment for Dykstra's Algorithm}

\begin{figure}[h]
    \centering
    \includegraphics[width=0.4\textwidth]{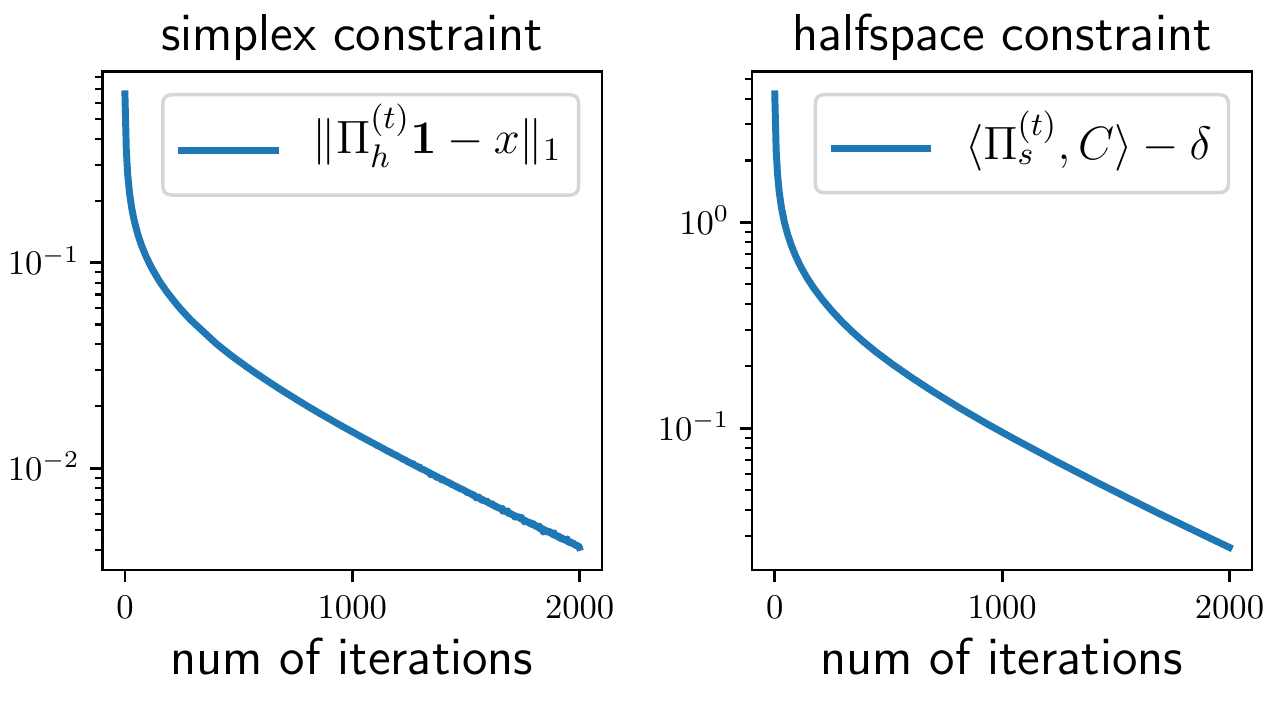}
    \caption{Convergence of Dykstra's algorithm. The violation of simplex constraint and halfspace constraint of $\Pi_h^{(t)}$ and $\Pi_s^{(t)}$ (the iterates produced by Dykstra's algorithm) gradually decrease to zero, but at a slow rate.}
    \label{fig:dykstra_convergence}
\end{figure}

We randomly sample a vector $\xv \in \Rb^{100}$ and a coupling $\Pi \in \Rb^{100 \times 100}$ (both from a uniform distribution in a hypercube). We normalize $\xv$ and $\Pi$.
We then project $\Pi$ to $\Cc_s \cap \Cc_h$. We set $\delta = 1$ and the cost matrix $C$ is the same as the one in \S\ref{sec:exp} (we reshape $\xv$ into a $10 \times 10$ matrix and view it as an image).
The convergence plots are shown in \Cref{fig:dykstra_convergence}. Dykstra's algorithm does converge, but at a slow rate.

\section{Toy Example: Failure of Dual Linear Minimization without Entropic Regularization}
\label{sec:appendix_failure_lmo}


This section presents a simple example to demonstrate the failure of dual LMO without adding entropic regularization.

Let $\delta = 0.5$, $\xv = (1, 0)^\top$. Let
\begin{align*}
G = 
\begin{pmatrix}
    1 & -1 \\
    0 & 0 
\end{pmatrix}
,
~ ~ ~ ~
C = 
\begin{pmatrix}
    0 & 1 \\
    1 & 0 
\end{pmatrix}.
\end{align*}
Let
\begin{align*}
\Pi = 
\begin{pmatrix}
    \Pi_{11} & \Pi_{12} \\
    \Pi_{21} & \Pi_{22} 
\end{pmatrix}.
\end{align*}
Then the primal linear program is
\begin{align*}
    & \mini_{\Pi_{11}, \Pi_{12} \geq 0} ~~ \Pi_{11} - \Pi_{12} \\
    & \sbjto ~ \Pi_{11} + \Pi_{12} = 1 , \Pi_{12} \leq 0.5
\end{align*}
It is easy to check that the solution is
\begin{align*}
\Pi^\star = 
\begin{pmatrix}
    0.5 & 0.5 \\
    0 & 0 
\end{pmatrix}.
\end{align*}
The dual linear program is
\begin{align*}
    \maxi_{\lambda \geq 0} - \frac12 \lambda + \min\left\{1, -1 + \lambda\right\}
\end{align*}
It is easy to see that the dual problem has a unique solution $\lambda^\star = 2$. Now we try to use the following condition to recover the primal solution:
\begin{align}
\label{eq:lp_stationary_condition}
    \Pi^\star \in \argmin_{\Pi \geq 0, \Pi \one = \xv} ~ \langle \Pi, G + \lambda^\star C \rangle - \lambda^\star \delta,
\end{align}
which is equivalent to
\begin{align}
\label{eq:lp_stationary_condition_example}
    \Pi^\star \in \argmin_{\Pi \geq 0, \Pi \one = \xv} ~
    \left\langle \Pi, 
    \begin{pmatrix}
    1 & 1 \\
    2 & 0 
    \end{pmatrix} 
    \right\rangle.
\end{align}
But it turns out that any $\Pi$ of the form
\begin{align*}
\Pi(\alpha) =
\begin{pmatrix}
    \alpha & 1 - \alpha \\
    0 & 0 
\end{pmatrix},
\end{align*}
where $0 \leq \alpha \leq 1$, is a minimizer. By varying $\alpha$, $\Pi(\alpha)$ can be suboptimal ($\alpha = 0$), optimal ($\alpha = 0.5$) or even infeasible ($\alpha = 1$). Thus, $\Pi^\star$ cannot be recovered by only considering the stationary condition.

Of course it is possible to combine \eqref{eq:lp_stationary_condition} with other KKT conditions (specifically, complementary slackness and primal feasibility) to obtain one of the primal solutions.
Particularly, in the above example, \eqref{eq:lp_stationary_condition} along with the complementary slackness determines the unique primal solution $\Pi^\star$. However, there are still two issues.
The first issue is that in more general cases, doing so requires solving a linear system whose variables are from a subset of $\Pi$, which could be GPU unfriendly.
More critically, the above solution is numerically unstable.
Suppose that there is a slight numerical inaccuracy due to floating point precision, such that \eqref{eq:lp_stationary_condition_example} becomes
\begin{align}
\label{eq:lp_stationary_condition_numerical_example}
    \Pi^\star \in \argmin_{\Pi \geq 0, \Pi \one = \xv} ~
    \left\langle \Pi, 
    \begin{pmatrix}
    1 + \xi &  1 - \xi \\
    2 & 0 
    \end{pmatrix} 
    \right\rangle.
\end{align}
for some small constant $\xi > 0$.
Now solving \eqref{eq:lp_stationary_condition_numerical_example} gives
\begin{align*}
\Pi(1) =
\begin{pmatrix}
    0 & 1 \\
    0 & 0 
\end{pmatrix},
\end{align*}
which is infeasible.

\section{Proofs}

This section presents all proofs in the paper.
Recall that without loss of generality we assume that all entries in the cost matrix $C$ are nonnegative and only diagonal entries of $C$ are zeros.
All proofs below assume it implicitly.


\subsection{Dual Projection}


\projectiondualobjective*

\begin{proof}
Introducing the Lagrange multiplier $\lambda \geq 0$ for the constraint $\langle \Pi, C \rangle \leq \delta$, we arrive at the following dual problem
\begin{align*}
    \maxi_{\lambda \geq 0} ~ g(\lambda),
\end{align*}
where
\begin{align*}
    g(\lambda) = \min_{\Pi \one = \xv, \Pi \geq 0} ~ \frac{1}{2} \| \Pi - G \|_{\mathrm{F}}^2 + \lambda \left(\langle \Pi, C \rangle - \delta \right).
\end{align*}
We complete the square in the inner problem, which leads to
\begin{align*}
    g(\lambda) = \min_{\Pi \one = \xv, \Pi \geq 0} ~ & \frac{1}{2} \| \Pi - \left(G - \lambda C\right) \|_{\mathrm{F}}^2 - \frac12 \lambda^2 \| C \|_{\mathrm{F}}^2 + \lambda \langle G, C \rangle - \lambda \delta.
\end{align*}
Notice that the constraint in the minimization is independent for each row of $\Pi$.
Thus, it can be reduced to a simplex projection for each row of $G - \lambda C$, which can be solved in $O(n^2 \log n)$ time.

By Danskin's theorem, $g$ is differentiable and the derivative is
\begin{align*}
    g^\prime(\tilde \lambda) = \langle \tilde\Pi, C \rangle - \delta,
\end{align*}
given the solution $\tilde \Pi$ to the minimization problem.
\end{proof}

Before proving \Cref{prop:projection_dual_upper_bound}, we first prove \Cref{lma:projection_simple_case}, which characterizes the solution of simplex projection in a special case.
Intuitively, the projection of a vector to a simplex is very sparse if one of its entries is much larger than the others.

\begin{lemma}
\label{lma:projection_simple_case}
Consider the following projection of a vector $\vv$ to a simplex:
\begin{align*}
    & \mini_{\wv \in \Rb^{n}} ~~ \norm{\wv - \vv}_2^2 \\
    & \sbjto ~ \sum_{i = 1}^{n} w_i = z, ~ w_i \geq 0,
\end{align*}
where $z > 0$.
Suppose that there exists $i$ such that $v_i \geq v_j + z$ for all $j \neq i$.
Then the solution is
$\wv^\star = (0, \cdots 0, z, 0, \cdots, 0)^\top$,
where the only nonzero entry is $w_i^\star = z$. 
\end{lemma}
\begin{proof}
A careful analysis of the simplex projection algorithm \citep{DuchiSSC08} would give a proof. Here, we give an alternative simple proof that does not rely on the algorithm. Assume to the contrary that there exists $j \neq i$ such that $w^\star_{j} > 0$ hence also $w^\star_i < z$. We construct another feasible point $\hat\wv$ by
\begin{align*}
    \hat w_i & = w_i^\star + w_j^\star \\
    \hat w_j & = 0 \\
    \hat w_k & = w_k^\star ~ ~ \forall k \neq i, k \neq j.
\end{align*}
Comparing the objective value of $\wv^\star$ and $\hat\wv$, we have
\begin{align*}
    \norm{\wv^\star - \vv}_2^2 - \norm{\hat\wv - \vv}_2^2
    & = (w_i^\star - v_i)^2 + (w_j^\star - v_j)^2 - (w_i^\star + w_j^\star - v_i)^2 - (0 - v_j)^2 \\
    & = 2 w_j^\star (v_i - v_j - w_i^\star) \\
    & > 2 w_j^\star (v_i - v_j - z) \\
    & \geq 0.
\end{align*}
$\hat\wv$ has even smaller objective value than $\wv^\star$, contradicting the optimality of $\wv^\star$.
Thus all $w_j^\star = 0$ for all $j \neq i$, which finishes the proof.
\end{proof}

\projectiondualupperbound*

\begin{proof}
By Danskin's theorem the dual problem \eqref{eq:projection_dual_problem} is differentiable in $\lambda$.
Moreover, for any given $\tilde\lambda$, suppose the solution to the minimization is $\tilde\Pi$, then the gradient \wrt $\tilde\lambda$ is $\langle \tilde\Pi, C \rangle - \delta$.

Consider the $i$-th row of $G - \lambda C$. Assume on the contrary that
\begin{align*}
    \lambda > \frac{2 \norm{\mathrm{vec}(G)}_\infty + \norm{\xv}_\infty}{\min_{i \neq j} \{C_{ij}\}}.
\end{align*}
Then, for all $i \neq j$ we have (recall that $C_{ii} = 0$)
\begin{align*}
    G_{ii} = G_{ii} - \lambda C_{ii} > G_{ij} - \lambda C_{ij} + x_i.
\end{align*}
The condition in \Cref{lma:projection_simple_case} is satisfied. A projection of $G - \lambda C$ results in a diagonal matrix $\tilde\Pi$. Thus
\begin{align*}
    g^\prime(\tilde \lambda) & = \langle \tilde\Pi, C \rangle - \delta \\
                                                  & = \sum_{i = j} \tilde\Pi_{ij} C_{ij} + \sum_{i \neq j} \tilde\Pi_{ij} C_{ij} - \delta \\
                                                  & = - \delta.
\end{align*}
The derivative is strictly negative, hence $\tilde \lambda$ is suboptimal, which finishes the proof.

\end{proof}

\projectiondualrecover*
\begin{proof}
This is a direct implication of KKT conditions.
\end{proof}

\subsection{Dual Linear Minimization Oracle without Entropic Regularization}


\frankdualobjective*
\begin{proof}
Introducing the Lagrange multiplier $\lambda \geq 0$ for the constraint $\langle \Pi, C \rangle \leq \delta$, we arrive at the following dual problem
\begin{align*}
    \maxi_{\lambda \geq 0} ~ g(\lambda),
\end{align*}
where
\begin{align*}
    g(\lambda) & = \min_{\Pi \geq 0, \Pi \one = \xv} ~ \langle \Pi, H \rangle + \lambda \left(\langle \Pi, C \rangle - \delta\right) \\
    & = \min_{\Pi \geq 0, \Pi \one = \xv} ~ \langle \Pi, H + \lambda C \rangle - \lambda\delta \\
    & = - \lambda \delta + \sum_{i = 1}^{n} x_i \min_{1 \leq j \leq n} \Big(H_{ij} + \lambda C_{ij}\Big).
\end{align*}
The last equality uses the fact that the constraints are independent for each row, thus the minimization is separable.
\end{proof}

\frankdualupperbound*
\begin{proof}
A key observation is that the dual objective is a piece-wise linear function \wrt $\lambda$. We can roughly estimate the range of the maximizer, by analyzing the slope of this function. 

Suppose
\begin{align}
\label{eq:frank_lmo_contrary_upper_bound}
    \lambda > \frac{2 \norm{\mathrm{vec}(H)}_\infty}{\min_{i \neq j} \left\{C_{ij}\right\}}.
\end{align}
Then for all $i \neq j$, we have
\begin{align*}
    \lambda C_{ij} > H_{ii} - H_{ij},
\end{align*}
which implies $H_{ii} + \lambda C_{ii} < H_{ij} + \lambda C_{ij}$ for all $i \neq j$ (recall that $C_{ii} = 0$).
Thus 
\begin{align*}
    g(\lambda) & = - \lambda \delta + \sum_{i = 1}^{n} x_i \min_{1 \leq j \leq n} \Big(H_{ij} + \lambda C_{ij}\Big) \\
               & = - \lambda \delta + \sum_{i = 1}^{n} x_i \left(H_{ii} + \lambda C_{ii}\right) \\
               & = - \lambda \delta + \sum_{i = 1}^{n} x_i H_{ii}
\end{align*}
is a linear function with negative slope.
Thus, any $\lambda$ satisfies \eqref{eq:frank_lmo_contrary_upper_bound} cannot be a dual solution, which completes the proof.
\end{proof}

\subsection{Dual Linear Minimization Oracle with Dual Entropic Regularization}


\frankwolfeentropydualproblem*
\begin{proof}
Introducing the Lagrange multiplier $\lambda \geq 0$ for the constraint $\langle \Pi, C \rangle \leq \delta$, we arrive at the following dual problem
\begin{align*}
    \maxi_{\lambda \geq 0} ~ g(\lambda),
\end{align*}
where
\begin{align}
    g(\lambda) & = \min_{\Pi \geq 0, \Pi \one = \xv} \langle \Pi, H \rangle + \gamma \sum_{i = 1}^{n}\sum_{j = 1}^{n} \Pi_{ij} \log \Pi_{ij} + \lambda (\langle \Pi, C \rangle - \delta) \\
\label{eq:frank_kl_projection}
               & = \min_{\Pi \geq 0, \Pi \one = \xv} - \lambda \delta + \langle \Pi, H + \lambda C \rangle + \gamma \sum_{i = 1}^{n}\sum_{j = 1}^{n} \Pi_{ij} \log \Pi_{ij}.
\end{align}
Without loss of generality, we assume that $x_i$ is strictly positive for all $i$.
Otherwise, $x_i = 0$ implies $\Pi_{ij} = 0$ for all $j$, thus the $i$-th row of $\Pi$ does not even appear in the minimization.

Notice that the inner minimization in \eqref{eq:frank_kl_projection} is separable, since the constraint on $\Pi$ is independent for each row.
For each row, the minimization is equivalent to a Kullback–Leibler projection to a simplex, which admits a closed form.
For the sake of completeness, we give a derivation here.
For the $i$-th row,
\begin{align*}
    \sum_{j = 1}^{n} \Pi_{ij} \left(H_{ij} + \lambda C_{ij}\right) + \gamma \sum_{j = 1}^{n} \Pi_{ij} \log \Pi_{ij} & = \gamma \sum_{j = 1}^{n} \Pi_{ij} \log\frac{\Pi_{ij}}{\exp\left(- \frac{H_{ij} + \lambda C_{ij}}{\gamma}\right)} \\
    & = \gamma \sum_{j = 1}^{n} \Pi_{ij} \left( \log\frac{\Pi_{ij} / x_i}{\exp\left(- \frac{H_{ij} + \lambda C_{ij}}{\gamma}\right) / a} + \log x_i - \log a \right) \\
    & = \gamma \sum_{j = 1}^{n} \Pi_{ij} \log\frac{\Pi_{ij} / x_i}{\exp\left(- \frac{H_{ij} + \lambda C_{ij}}{\gamma}\right) / a} + \gamma x_i \left(\log x_i - \log a \right) \\
    & \geq \gamma x_i \left(\log x_i - \log a \right),
\end{align*}
where $a = \sum_{j = 1}^{n} \exp\left(- \frac{H_{ij} + \lambda C_{ij}}{\gamma}\right)$ is a normalization constant. The last inequality holds if and only if
\begin{align*}
    \Pi_{ij} = x_i \frac{\exp\left(- \frac{H_{ij} + \lambda C_{ij}}{\gamma}\right)}{\sum_{j = 1}^{n} \exp\left(- \frac{H_{ij} + \lambda C_{ij}}{\gamma}\right)}.
\end{align*}
Plugging in the above expression finishes the proof.
\end{proof}

\frankwolfeentropydualrecover*
\begin{proof}
This is a direct implication of KKT conditions. See the KL projection derivation in the proof of \Cref{prop:frank_wolfe_entropy_dual_problem} for a detailed explanation.
\end{proof}

\frankentropydualupperbound*
\begin{proof}
To begin with, we have the following bound on the derivative:
\begin{align*}
    g^\prime(\lambda)
    & = -\delta + \sum_{i=1}^n x_i \frac{\sum_{j=1}^n \exp\left(- \frac{H_{ij} + \lambda C_{ij}}{\gamma}\right) C_{ij}}{\sum_{j=1}^{n} \exp\left(- \frac{H_{ij} + \lambda C_{ij}}{\gamma}\right)} \\
    & = -\delta + \sum_{i=1}^n x_i \frac{\sum_{j=1}^n \exp\left(- \frac{H_{ij} + \lambda C_{ij} - H_{ii}}{\gamma}\right) C_{ij}}{\sum_{j = 1}^{n} \exp\left(- \frac{H_{ij} + \lambda C_{ij} - H_{ii} }{\gamma}\right)} \\
    & = -\delta + \sum_{i=1}^n x_i \frac{\sum_{j \neq i} \exp\left(- \frac{H_{ij} + \lambda C_{ij} - H_{ii}}{\gamma}\right) C_{ij}}{1 + \sum_{j \neq i} \exp\left(- \frac{H_{ij} + \lambda C_{ij} - H_{ii} }{\gamma}\right)} \\
    & \leq -\delta + \sum_{i = 1}^n \sum_{j \neq i} x_i \exp\left(- \frac{H_{ij} + \lambda C_{ij} - H_{ii}}{\gamma}\right) C_{ij}.
\end{align*}
The first equality uses translation invariance property of softmin function.
The last inequality uses the fact that $C_{ii} = 0$.
Notice that
\begin{align*}
    \lambda > \frac{ \left[2 \norm{\mathrm{vec}(H)}_\infty + \gamma \log \left(\frac{1}{\delta}\xv^\top C \one \right)\right]_+}{\min_{i \neq j}\left\{ C_{ij} \right\}}
\end{align*}
implies
\begin{align*}
\label{eq:implication_negation_proposition}
    \lambda C_{ij} & > 2 \norm{\mathrm{vec}(H)}_\infty + \gamma \log\left( \frac{1}{\delta}\xv^\top C \one \right) \\
                   & \geq H_{ii} - H_{ij} + \gamma \log\left( \frac{1}{\delta} \xv^\top C \one\right),
\end{align*}
for all $i \neq j$.
Thus, we have
\begin{align*}
    \exp\left(- \frac{H_{ij} + \lambda C_{ij} - H_{ii}}{\gamma} \right) < \frac{\delta}{\xv^\top C \one}
\end{align*}
for all $i \neq j$.
Plug it back to $g^\prime(\lambda)$.
We have
\begin{align*}
    g^\prime(\lambda)
    & < - \delta + \frac{\delta}{\xv^\top C \one}  \cdot \sum_{i=1}^n \sum_{j = 1}^{n} x_i C_{ij} \\
    & = - \delta + \frac{\delta}{\xv^\top C \one}  \cdot  \xv^\top C \one \\
    & = 0.
\end{align*}
The derivative is strictly negative hence $\lambda$ cannot be optimal, which concludes the proof.
\end{proof}

\clearpage

\section{Further Experimental Details}
\label{sec:appendix_exp_setting}


In this section, we present further experimental details as well as some implementation details.

\subsection{Models}

Our MNIST and CIFAR-10, models are taken from \citep{WongSK19}. The MNIST model is a  convolutional network with ReLU activations which achieves $98.89\%$ clean accuracy. The CIFAR-10 model is a residual network with $94.76\%$ clean accuracy.
The ImageNet model is a ResNet-50 pretrained neural network, downloaded from PyTorch models subpackage, which achieves $72.0\%$ top-1 clean accuracy on the first 100 samples from the validation set.
All experiments are run on a single P100 GPU.

\subsection{Stopping Criteria of Projection and Linear Minimization Step}
\label{sec:stop_criterion}

\textbf{Stopping criterion of projected Sinkhorn}
Denote $\textrm{obj}^{(t)}$ as the dual objective value of projected Sinkhorn in $t$-th iteration, we stop the algorithm upon the following condition is satisfied:
\begin{align*}
    \lvert \textrm{obj}^{(t + 1)} - \textrm{obj}^{(t)} \rvert \leq 10^{-4} + 10^{-4} \cdot \textrm{obj}^{(t)},
\end{align*}
which is also used by \citet{WongSK19}.

\textbf{Stopping criterion of dual projection and dual LMO} Both dual projection and dual LMO use the bisection method to solve dual problems. Bisection is terminated upon either
\begin{align*}
    u - l \leq 10^{-4} ~~ \text{or} ~~ \lvert g^\prime(\tilde\lambda) \rvert \leq 10^{-4}
\end{align*}
This condition lets us determine the $4$-th digit after the decimal point of $\lambda^\star$, or the violation of transportation cost constraint is less than $10^{-4}$. Note that a violation of $10^{-4}$ is extremely small, compared with $\delta = \epsilon \sum_{i = 1}^{n} \xv_i$, which is much (usually at least $10^{5}$ times) larger than the tolerance since the pixel sum $\sum_{i = 1}^{n} \xv_i$ is usually a large number.

In practice, the upper bound \eqref{eq:projection_dual_upper_bound} and \eqref{eq:frank_wolfe_entropy_dual_upper_bound} are often between $2$ and $3$ thanks to gradient normalization.
Thus, the bisection method satisfies the stopping criterion in at most $15$ iterations ($ 2 \times 2^{-15} \approx 10^{-4}$).

\subsection{Step Sizes of PGD}

\textbf{PGD with projected Sinkhorn} On MNIST, CIFAR-10 and ImageNet, the step sizes are set to $0.1$. Notice that $0.1$ is also the step size used by \citet{WongSK19} on MNIST and CIFAR-10. The gradient is normalized using $\ell_\infty$ norm:
\begin{align*}
    \argmax_{\norm{v}_\infty \leq 1} v^\top \nabla_{\xv} \ell(\xv, y) = \sign\left(\nabla_{\xv} \ell(\xv, y) \right).
\end{align*}
Again, this is the same setting used by \citet{WongSK19}.
While $\eta = 1.0$ achieves lower adversarial accuracy on the first batch of samples in \Cref{sec:appendix_loss_acc_vs_iter}, we find this large step size causes numerical overflow easily on the remaining batches.
Thus we choose $\eta = 0.1$ to present the experimental results.

\textbf{PGD with dual projection} On MNIST, the step size is set to $0.1$. On CIFAR-10 and ImageNet, the step size is set to $0.01$. The gradient is normalized in the following way:
\begin{align*}
    \frac{\nabla_{\Pi} \ell(\Pi^\top \one, y)}{\max_{i, j} \left\lvert \nabla_{\Pi} \ell(\Pi^\top \one, y) \right\rvert}.
\end{align*}

\subsection{Implementation of Local Transportation and Sparse Matrices Computation}
\label{sec:sparse_matrix_operation}

The local transportation technique in \S\ref{sec:local_transportation} requires computation on a sparse matrix $\Pi$.
However, a big challenge is that sparse matrices computation is not easily parallelizable on GPUs.
As such, current deep learning packages (PyTorch, TensorFlow) do not support general sparse matrices well.

To fully utilize GPU acceleration, we explore the sparsity pattern in $\Pi$.
Notice that each row of $\Pi$ has at most $k^2$ nonzero entries.
We store $\Pi$ as a $n \times k^2$ dense matrix, with some possible dummy entries.
The advantage is that now $\Pi$ is a dense matrix. Any row operations on $\Pi$ (\eg softmin along each row, sorting along each row) can be parallelized easily.
The downside, however, is that column operations (\eg summation over each column) might take extra efforts.  
Nevertheless, this is not a bottleneck of the speed of dual projection and dual LMO, since they only require efficient row operations.

\clearpage

\subsection{Further Results on Convergence of Outer Maximization}
\label{sec:appendix_loss_acc_vs_iter}

\begin{figure*}[h]
\centering
\begin{subfigure}{0.44\textwidth}
    \includegraphics[width=\textwidth]{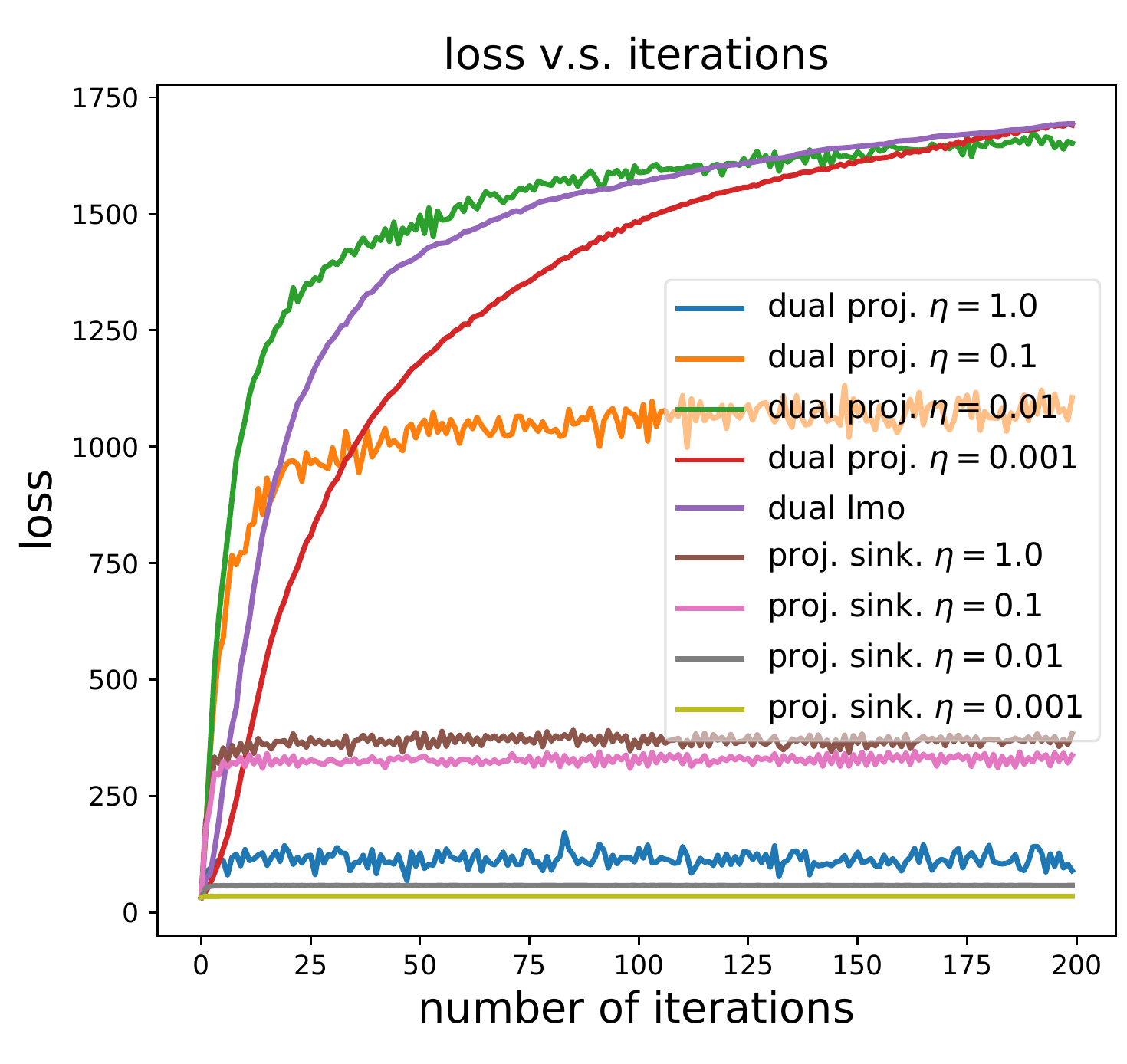}
    \caption{loss \wrt iterations on CIFAR-10}
\end{subfigure}
\begin{subfigure}{0.44\textwidth}
    \includegraphics[width=\textwidth]{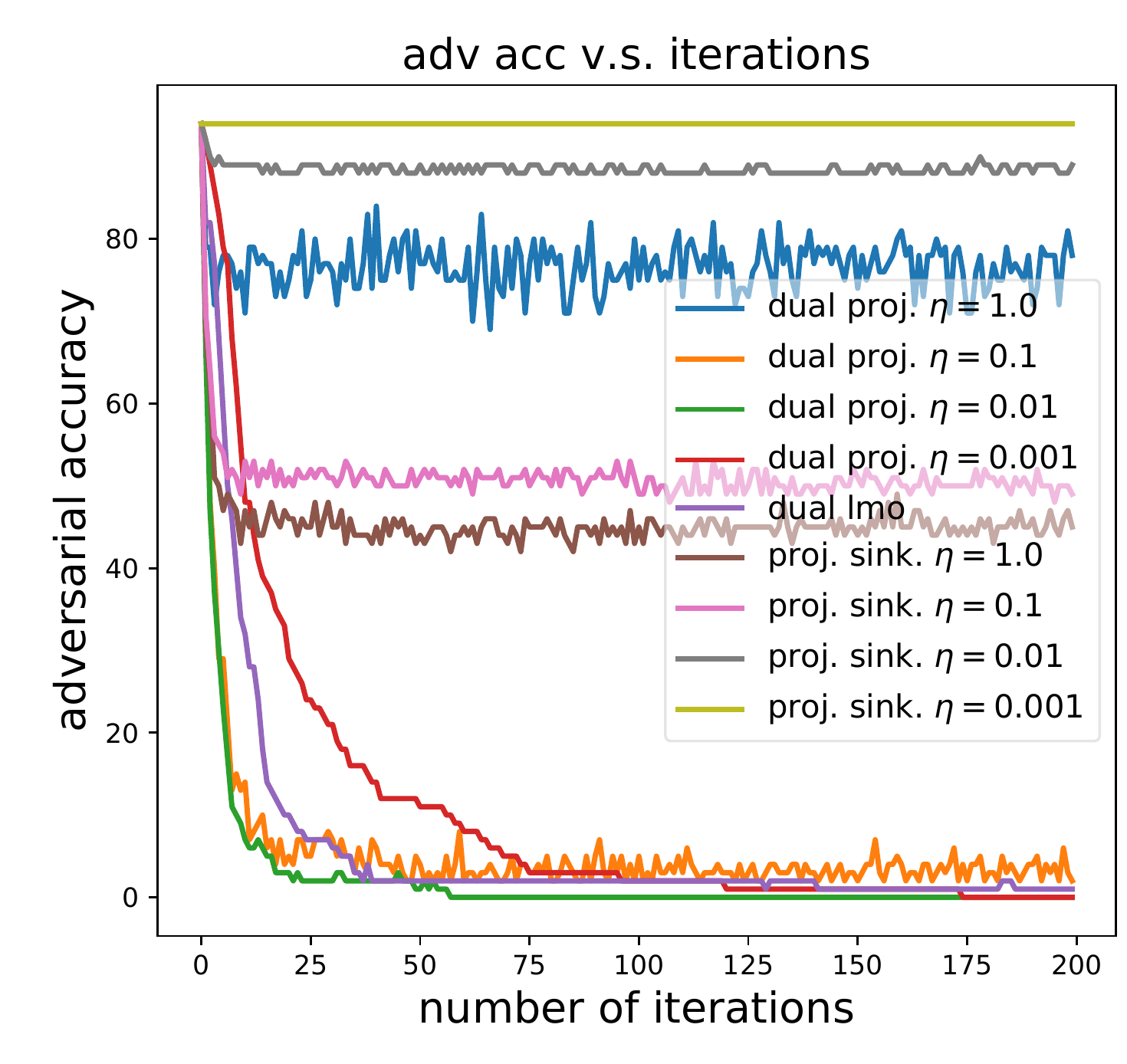}
    \caption{adversarial accuracy \wrt iterations on CIFAR-10}
\end{subfigure}

\begin{subfigure}{0.44\textwidth}
    \includegraphics[width=\textwidth]{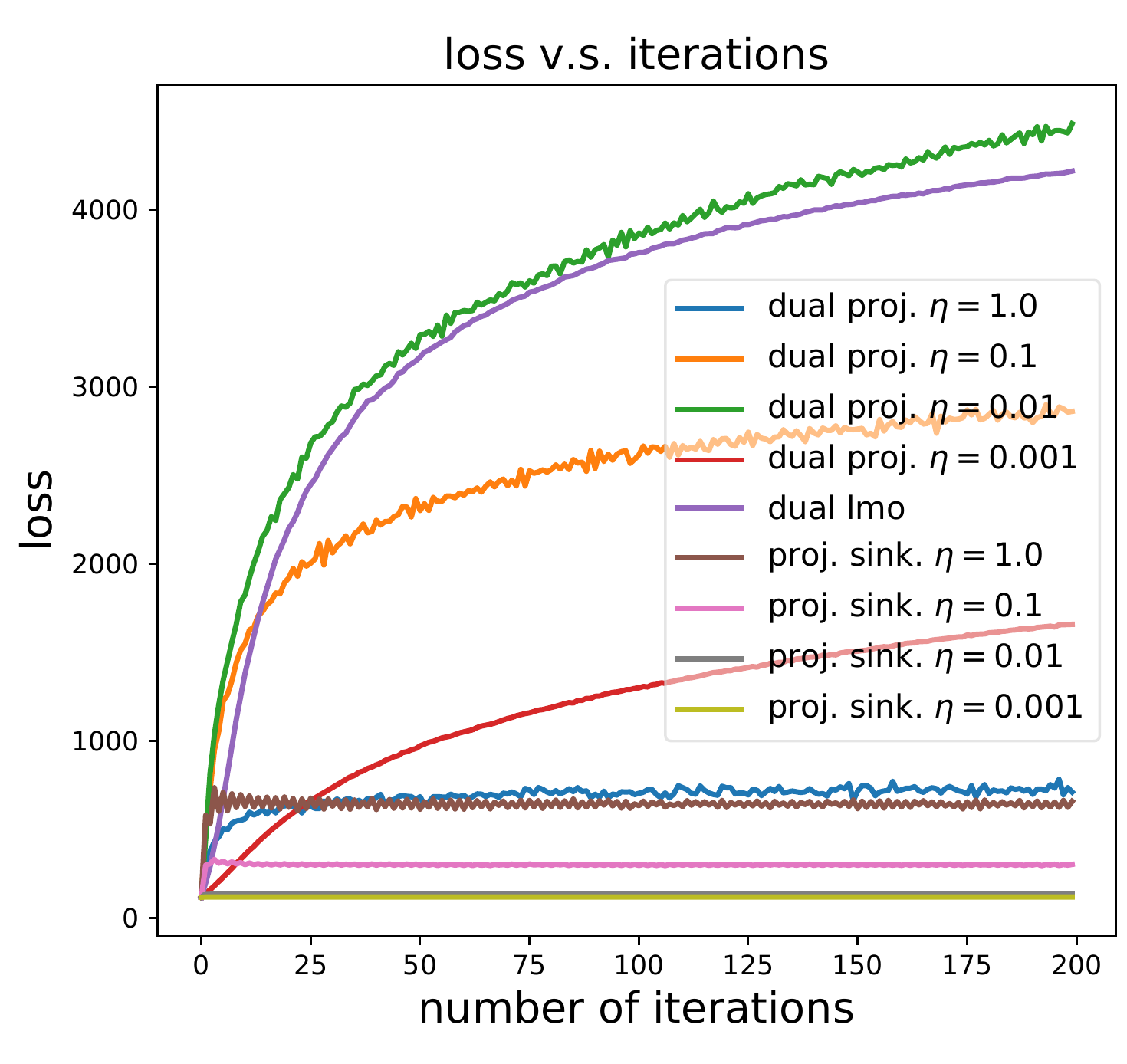}
    \caption{loss \wrt iterations on ImageNet}
\end{subfigure}
\begin{subfigure}{0.44\textwidth}
    \includegraphics[width=\textwidth]{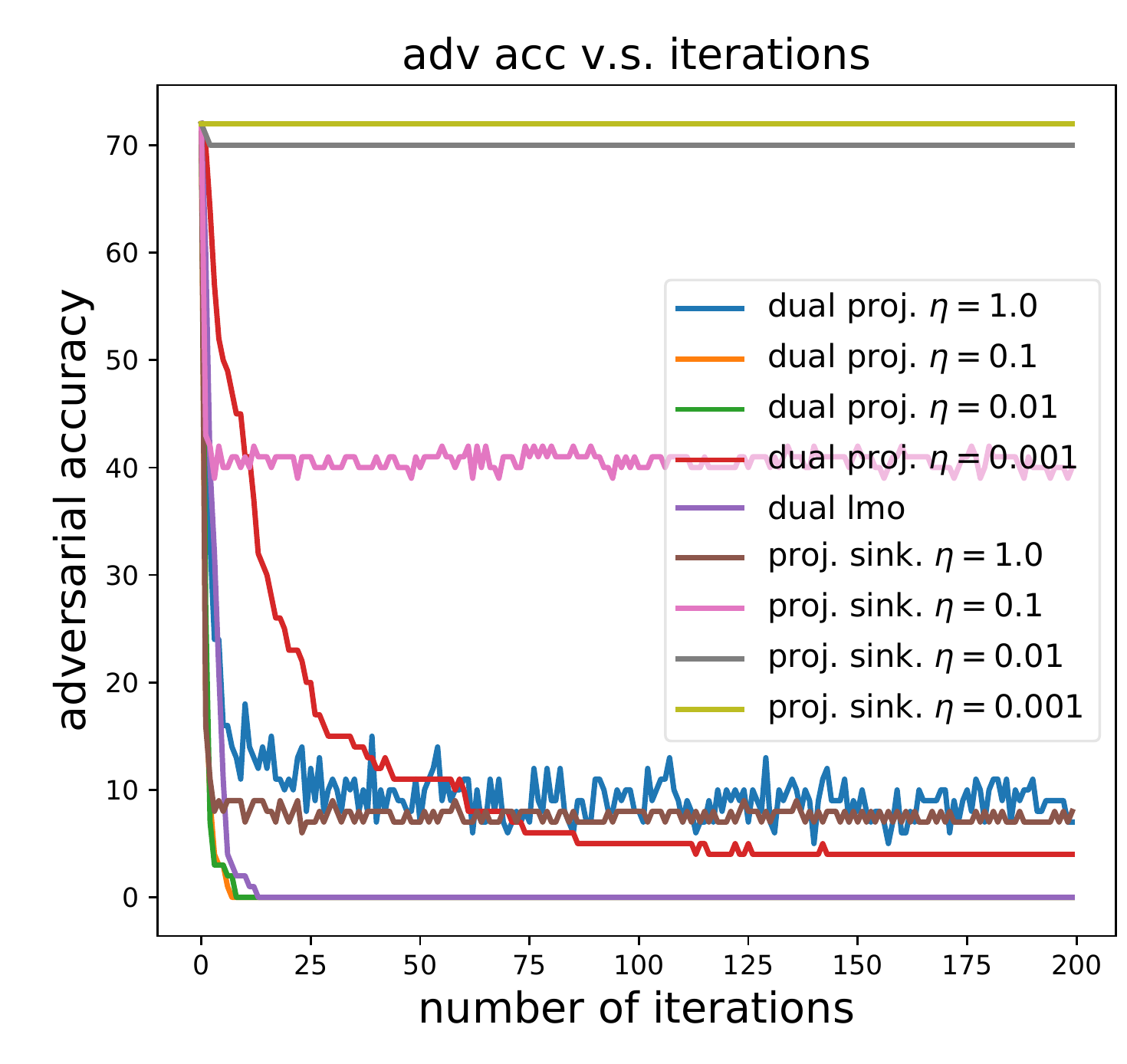}
    \caption{adversarial accuracy \wrt iterations on ImageNet}
\end{subfigure}
\caption{Convergence of outer maximization of different attacks.}
\label{fig:full_cifar_imagenet_converge}
\end{figure*}

We plot the loss and adversarial accuracy \wrt the number of iterations in \Cref{fig:full_cifar_imagenet_converge}
($\epsilon = 0.005$ on both CIFAR-10 and ImageNet).
Dual LMO uses $\gamma = 10^{-3}$.
Projected Sinkhorn uses $\gamma = 5 \cdot 10^{-5}$ on CIFAR-10 and $\gamma = 5 \cdot 10^{-6}$ on ImageNet.

\textbf{Frank-Wolfe with dual LMO} FW uses the default decay schedule $\frac{2}{t + 1}$. We observe that FW with dual LMO converges very fast especially at the initial stage, even when using the simple default decay schedule.

\textbf{PGD with Projected Sinkhorn}
We observe that when $\eta$ is small (\eg $\eta = 0.01, 0.001$), PGD with projected Sinkhorn barely makes progress in the optimization.
While aggressively large step sizes (\eg $1.0$ and $0.1$) can make progress, the curves are very noisy indicating the steps sizes are too large, and the loss is still much lower than the other two attacks.

\textbf{PGD with Dual Projection}
In contrast, PGD with dual projection has more meaningful curves: small $\eta$ (\eg $0.001$) converges very slowly; large $\eta$ (\eg $1.0$ and $0.1$) makes the curves noisy, while appropriate choice of $\eta$ (\eg $0.01$) always achieves the highest loss and also the lowest adversarial accuracy. In these cases, PGD with dual projection converges comparably and sometimes slightly faster than Frank-Wolfe.

\clearpage


\subsection{MNIST and CIFAR-10 Adversarial Examples}

\begin{figure}[h]
\centering

\begin{subfigure}{0.48\textwidth}
\begin{tikzpicture}
\node[inner sep=0] (adv_imgs) at (0, 0) {\includegraphics[width=\textwidth]{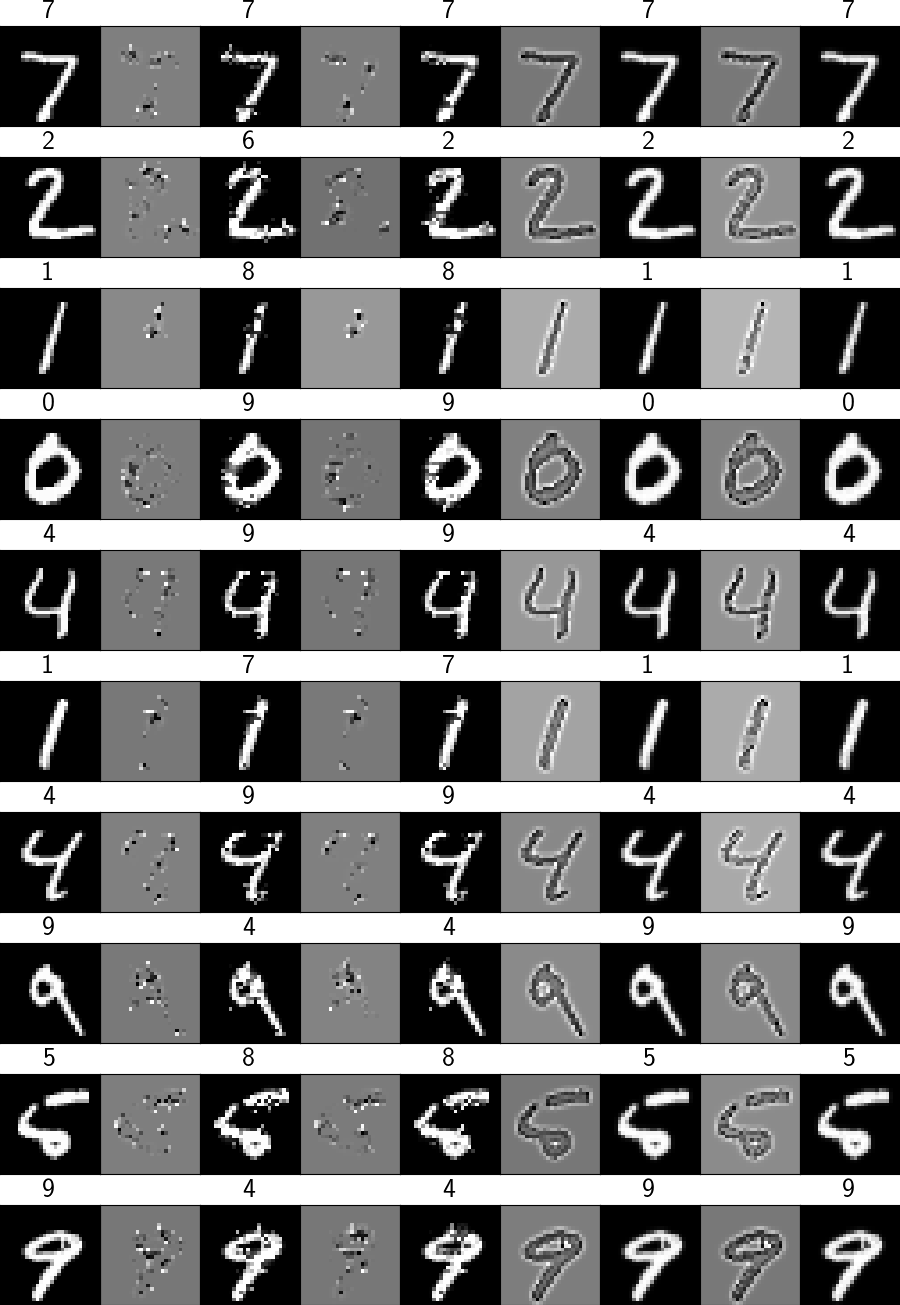}};

\getcorners{adv_imgs}

\drawtext{1}{0.888}{\scriptsize clean}{}
\drawtext{0.888}{0.666}{\scriptsize dual proj.}{}
\drawtext{0.666}{0.444}{\scriptsize dual LMO}{\scriptsize ($\gamma = 10^{-3}$)}
\drawtext{0.444}{0.222}{\scriptsize dual LMO}{\scriptsize ($\gamma = 10$)}
\drawtext{0.222}{0.0}{\scriptsize proj. Sink.}{\scriptsize ($\gamma = 1/1000$)}

\end{tikzpicture}

\caption{MNIST}

\end{subfigure}
\begin{subfigure}{0.48\textwidth}
\begin{tikzpicture}
\node[inner sep=0] (adv_imgs) at (0, 0) {\includegraphics[width=\textwidth]{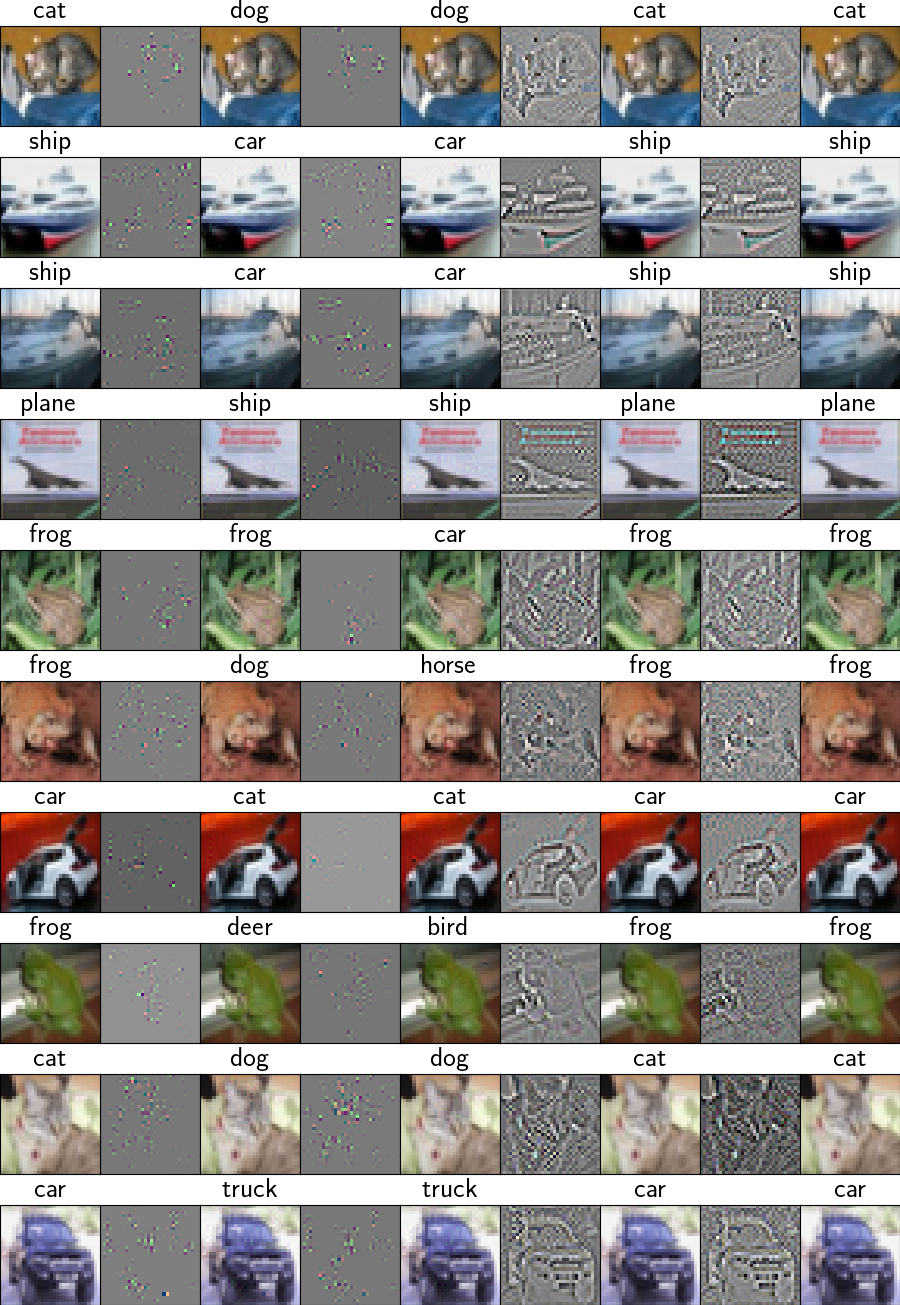}};

\getcorners{adv_imgs}

\drawtext{1}{0.888}{\scriptsize clean}{}
\drawtext{0.888}{0.666}{\scriptsize dual proj.}{}
\drawtext{0.666}{0.444}{\scriptsize dual LMO}{\scriptsize ($\gamma = 10^{-3}$)}
\drawtext{0.444}{0.222}{\scriptsize dual LMO}{\scriptsize ($\gamma = 10$)}
\drawtext{0.222}{0.0}{\scriptsize proj. Sink.}{\scriptsize ($\gamma = 1/3000$)}

\end{tikzpicture}

\caption{CIFAR-10}

\end{subfigure}

\caption{Comparison of Wasserstein adversarial examples and Wasserstein perturbations generated by different attacks on MNIST ($\epsilon = 0.2$) and CIFAR-10 ($\epsilon = 0.002$). Predicted labels are shown on the top of images. Perturbations are scaled linearly to $[0, 1]$ for visualization.}

\end{figure}

Wasserstein adversarial perturbations generated by PGD with dual projection and FW with dual LMO ($\gamma = 10^{-3}$) are very sparse.
Hence, they do not reflect the shapes in the clean images.
Perturbations reflect the shapes only when the entropic regularization introduces large approximation error (\eg FW with dual LMO ($\gamma = 10$) and PGD with projected Sinkhorn).

For the sake of visualization of the approximation error, we let PGD use smaller step sizes ($\eta = 0.01$) and more iterations (1000) when combined with projected Sinkhorn.\footnote{This is the same case for \Cref{fig:cifar_adv} in the main paper. Notice that our normalization method is slightly different from that of \textcite{WongSK19}, which might account for the slight difference of the visualization results.}
We observe that using large step size, \eg, $\eta=0.1$, often generates blurry perturbations (not necessarily reflecting shapes clearly).
But they are still much more dense compared with adversarial perturbations generated dual projection and dual LMO ($\gamma = 10^{-3}$).

\clearpage

\subsection{ImageNet Adversarial Examples}

\begin{figure}[!h]
\centering
\begin{subfigure}{0.46\textwidth}
    \includegraphics[width=\textwidth]{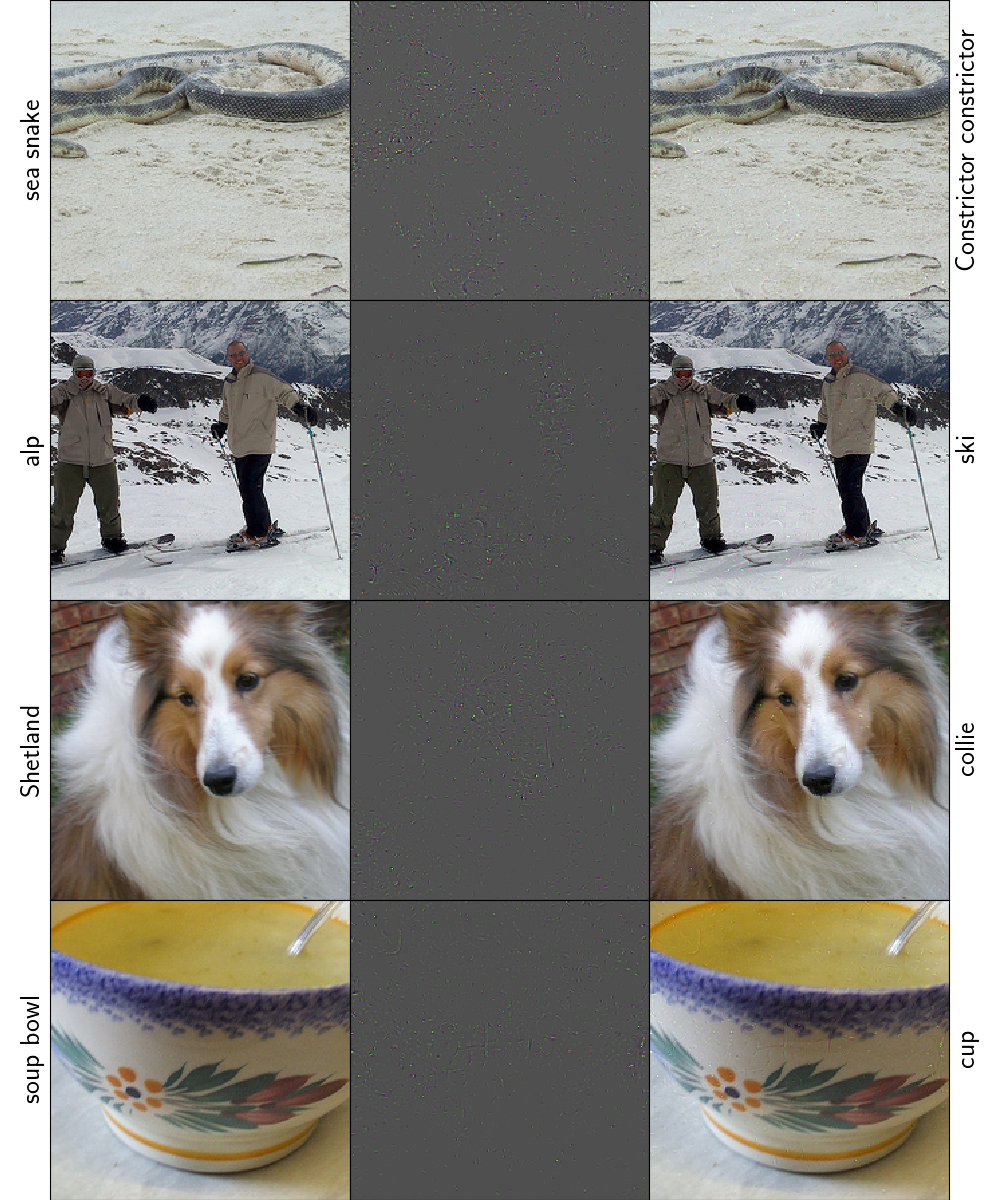}
    \caption{Dual projection}
\end{subfigure}
\begin{subfigure}{0.46\textwidth}
    \includegraphics[width=\textwidth]{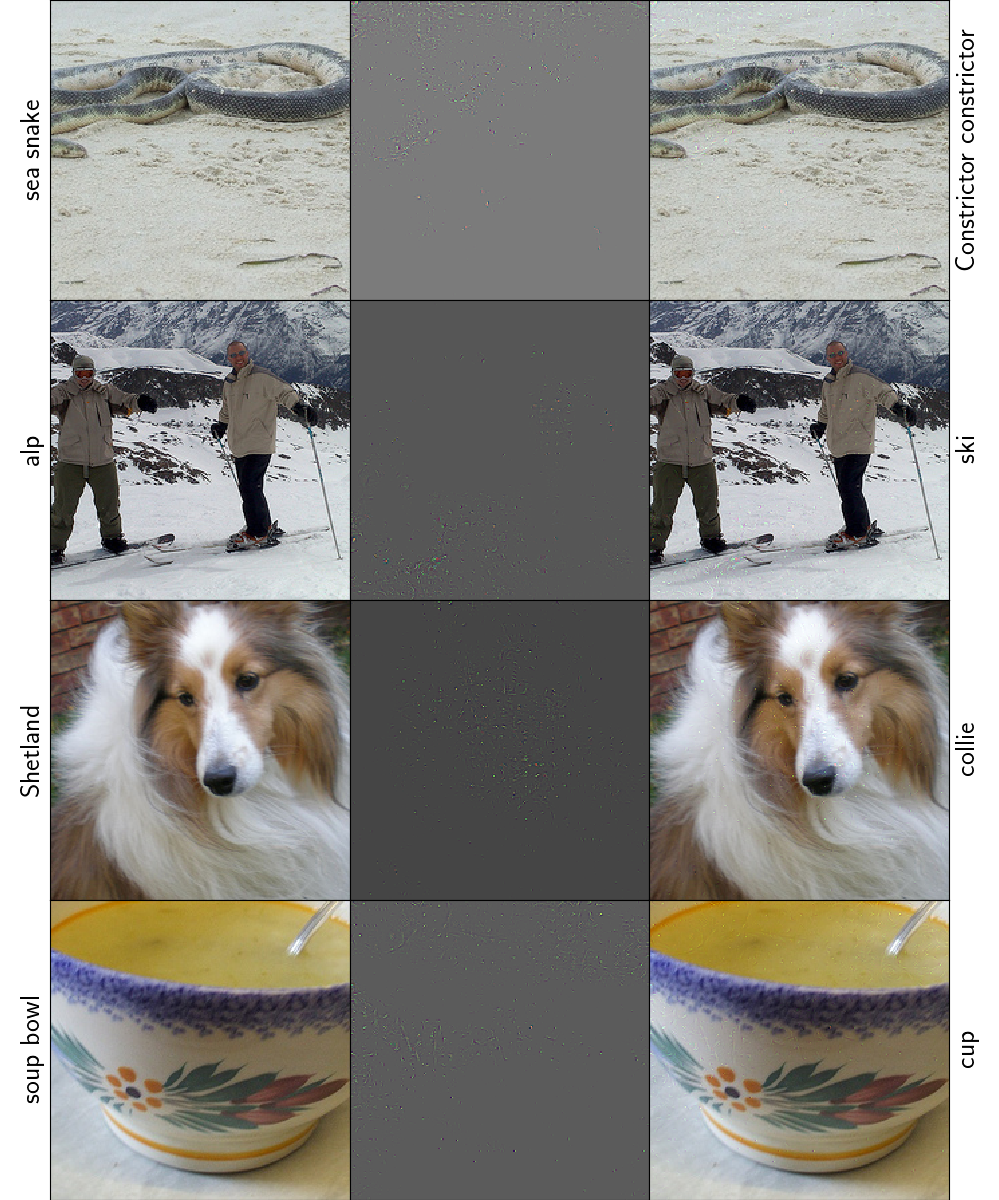}
    \caption{Dual LMO ($\gamma = 10^{-3}$)}
\end{subfigure}

\smallskip

\begin{subfigure}{0.46\textwidth}
    \includegraphics[width=\textwidth]{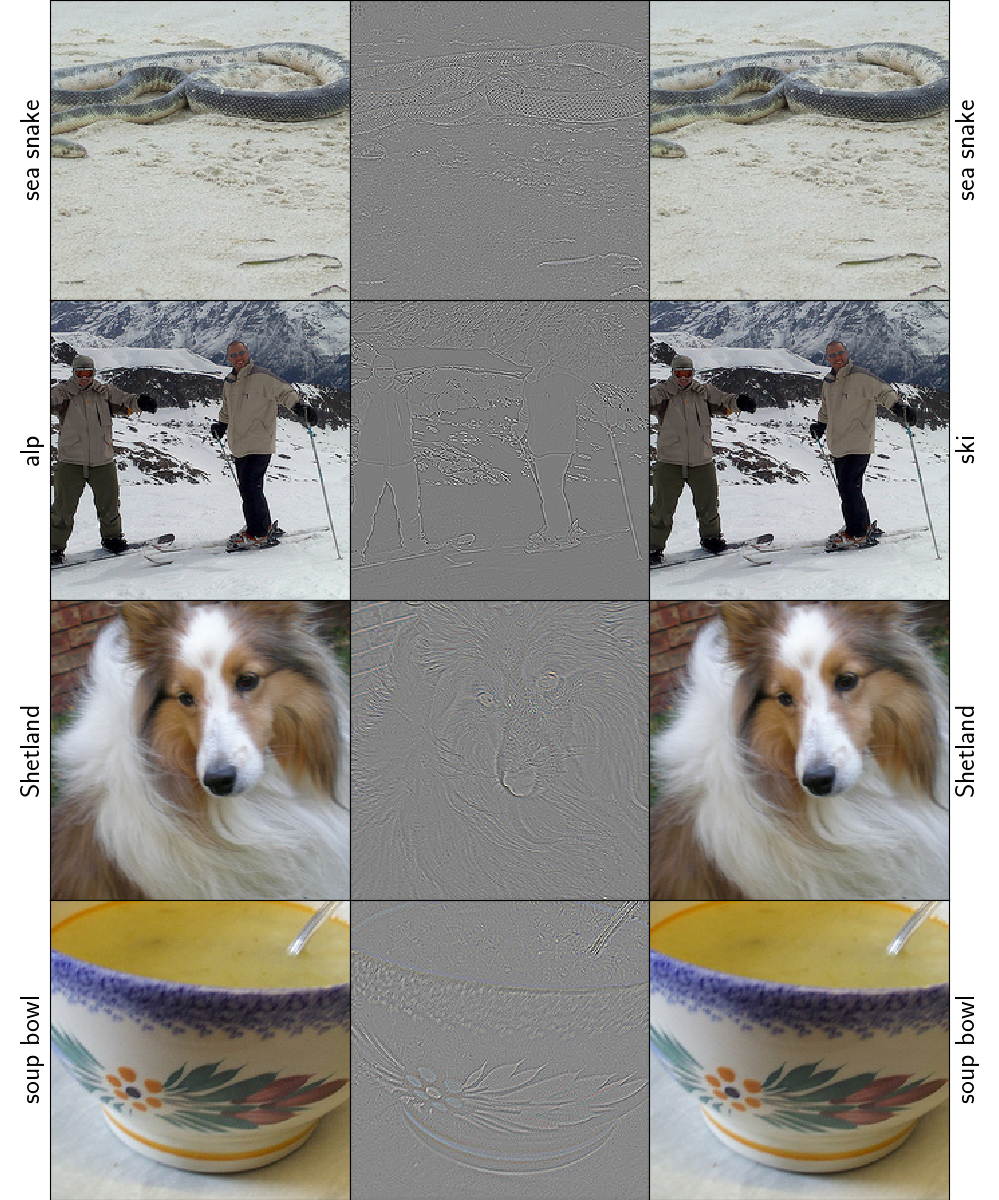}
    \caption{Dual LMO ($\gamma = 10$)}
\end{subfigure}
\begin{subfigure}{0.46\textwidth}
    \includegraphics[width=\textwidth]{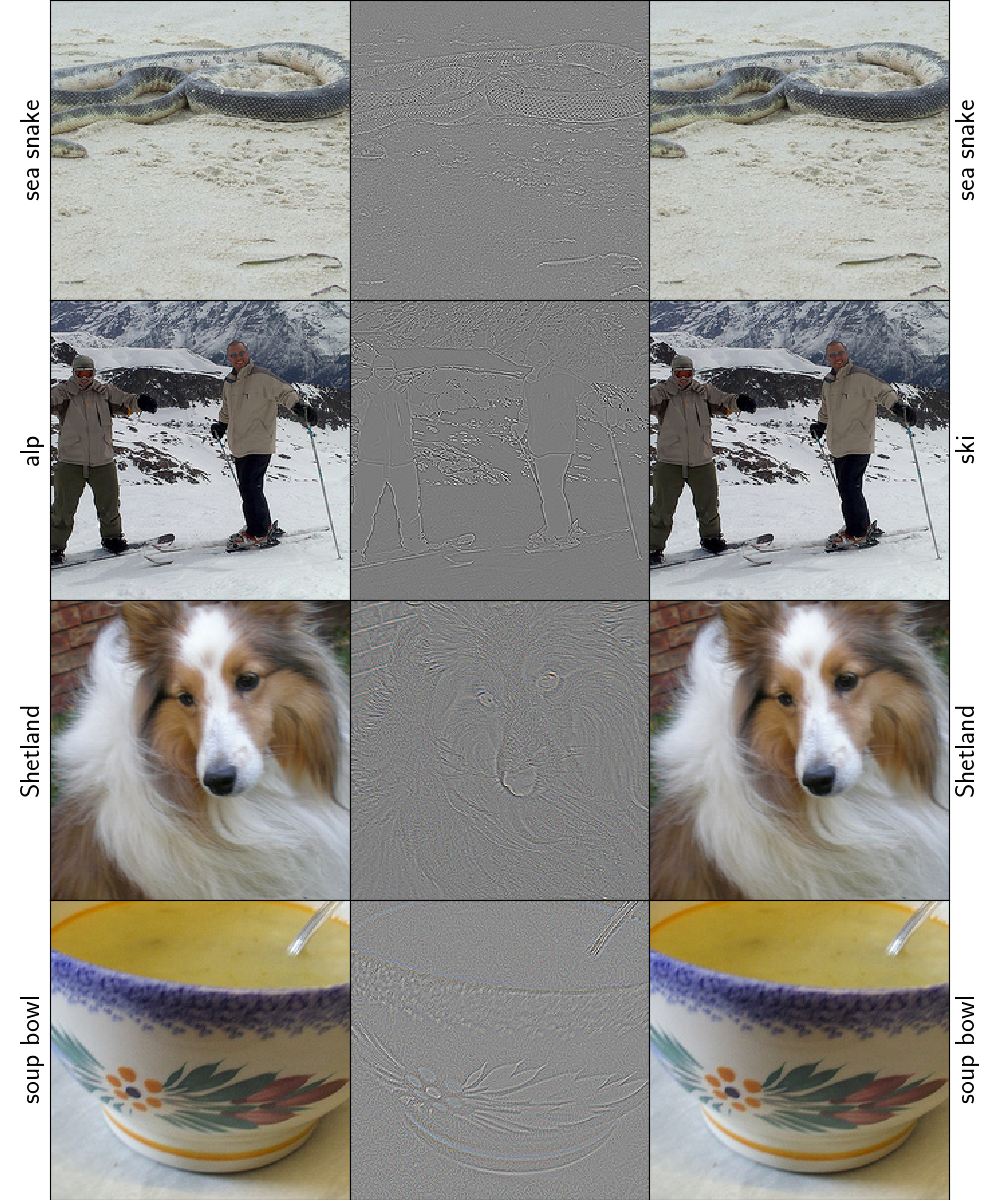}
    \caption{Projected Sinkhorn ($\gamma = 10^{-4}$)}
\end{subfigure}
\caption{Comparison of Wasserstein adversarial examples generated by different attacks ($\epsilon = 0.005$). Notice that adversarial pertubations reflect shapes in the images only when the entropic regularization is large (bottom two subfigures).}
\end{figure}

\subsection{Why Does Entropic Regularization Reflect Shapes in Images?}
\label{sec:appendix_entropy_relection_analysis}


A large entropic regularization term in the optimization objective (of projected Sinkhorn and dual LMO) encourages the transportation to be uniform, thus each pixel tends to spread its mass evenly to its nearby pixels.
In the region where pixel intensities do not change much, the transportations cancel out;
while in the region where pixel intensities change drastically (\eg, edges in an image), pixel mass flows from the high pixel region to low pixel region.
As a result, it reflects the edges in the original image.

\subsection{Analysis of Running Time of Projected Sinkhorn}
\label{sec:appendix_running_time_projected_sinkhorn}


One possible explanation for slow per iteration running speed of projected Sinkhorn is that, each iteration of projected Sinkhorn requires solving $n$ nonlinear univariate equations. In implementation, this step is done by calling Lambert function, which eventually calls Halley's method, a root-finding algorithm. While solving each nonlinear equation is counted as $O(1)$ in analysis, this operation is potentially much more expensive than other basic arithmetic operations (\eg addition and multiplication). In contrast, both dual projection and dual linear minimization only rely on more standard arithmetic operations (\eg multiplication, division and comparison, logarithm and exponential functions).

We test the running time of the nonlinear equation solvers in experiments. On a single RTX 2080 Ti GPU, for a batch of $100$ samples on MNIST, we observe that Lambert function evaluation takes about $26\%$ of the running time during each iteration of projected Sinkhorn. For a batch of $100$ samples on CIFAR-10, Lambert function evaluation takes about $28$\% of the running time during each iteration of projected Sinkhorn. For a batch of $50$ samples on ImageNet, Lambert function evaluation takes about $32\%$ of the running time of one iteration of projected Sinkhorn.

We note that GPU time recording is somewhat inconsistent on different machines.
In our case, all experiments in the main paper are run on a single P100 GPU on a cluster.
However, on a 2080 Ti local GPU, we observe that projected Sinkhorn is slightly faster than dual projection on ImageNet in terms of per iteration running time (but still at least twice slower than dual LMO).
While on MNIST and CIFAR-10, dual projection and dual LMO are consistently faster than projected Sinkhorn, both on P100 cluster and 2080 Ti local machine.

\clearpage

\subsection{Case Study: Feasibility of Generated Wasserstein Adversarial Examples}

In this section, we present a sanity check of feasibility of adversarial examples generated by different algorithms on MNIST.

\begin{table}[h]
\centering
\caption{
A sanity check of feasibility of Wasserstein adversarial examples generated by different algorithms on MNIST.
\textbf{2nd column:} The average Wasserstein distance between adversarial examples and clean images (the higher the better).
\textbf{3rd column:} The maximum pixel value in the generated adversarial examples (the lower the better).
\textbf{4th column:} The percentage of pixel mass that exceeds $1$ (the lower the better).
The third and the forth columns are largest values over a mini-batch, while the second column is the average over a mini-batch.
}
\label{tb:sanity_check} \begin{tabular}{l | c c c}
\toprule
\multicolumn{1}{c|}{\multirow{1}{*}{method}} &  $\Wc(\xv, \xv_{adv})$ & $\max_{i} \left\{ \xv_i \right\}$ & $\frac{\sum_{i = 1}^{n} \max\{\xv_i - 1, 0\}}{\sum_{i=1}^{n} \xv_i}$ \\
\midrule
PGD + Proj. Sink. ($\gamma = 1 / 1000$) & $0.109965$ & $1.474048$ & $3.646899\%$ \\
PGD + Dual Proj. ~(w/o post-processing) & $0.493146$ & $4.118621$ & $15.812986\%$ \\
PGD + Dual Proj. & $0.444885$ & $1.000030$ & $0.000034\%$ \\
FW ~ + Dual LMO (w/o post-processing) & $0.428238$ & $3.955514$ & $10.406417\%$ \\
FW ~ + Dual LMO & $0.399014$ & $1.000015$ & $0.000025\%$ \\
\bottomrule
\end{tabular}
\end{table}

\textbf{Setup}
We test all attacks on the first 100 samples on the test set of MNIST with $\epsilon = 0.5$.
All parameter settings are the same as the table in the main paper.
Dual LMO uses $\gamma = 10^{-3}$ and projected Sinkhorn uses $\gamma = 1 / 1000$.

\textbf{Wasserstein Constraint}
The (exact) Wasserstein distances presented in \Cref{tb:sanity_check} are calculated by a linear programming solver.
We observe that all adversarial examples strictly satisfy the the Wasserstein constraint $\Wc(\xv, \xv_{adv}) \leq 0.5$.
We also observe that the Wasserstein adversarial examples generated by PGD with dual projection and FW with dual LMO have much larger Wasserstein distance than those of projected Sinkhorn.
This again hightlights that projected Sinkhorn is only an approximate projection operator, thus PGD cannot fully explore the Wasserstein ball, resulting in a weak attack.

\textbf{Hypercube Constraint}
Without post-processing, all attacks generate Wasserstein adversarial examples that violate the hypercube constraint a lot.
However, after applying our post-processing algorithm, the generated Wasserstein adversarial examples roughly satisfy the hypercube constraint $[0, 1]^n$ up to a reasonable precision.

\clearpage

\subsection{Adversarially Trained Models}
\label{sec:adversarial_training}

In this section, we present additional experiments on adversarial training and attacking adversarially trained models.

\subsubsection{MNIST}
We adversarially train a robust model using Frank-Wolfe with dual LMO and a fixed perturbation budget $\epsilon=0.3$.
The inner maximization is approximated with $40$ iterations of Frank-Wolfe.\footnote{We also have tried larger perturbation budgets $\epsilon=0.4$ and $\epsilon=0.5$, but the training collapses.}
This model achieves $95.83\%$ clean accuracy.
The adversarial accuracy of this model is shown in \Cref{tab:mnist_adv_training_frank_wolfe}.

For comparison, we also present results on attacking an adversarially trained model by PGD with projected Sinkhorn.
We use a pretrained model released by \textcite{WongSK19}, which is adversarially trained using an adaptive perturbation budget $\epsilon \in [0.1, 2.1]$. 
This model achieves $97.28\%$ clean accuracy.

We notice that the model adversarially trained by PGD with projected Sinkhorn seems to overfit to the same attack.
Compared with the standard trained model in \Cref{tb:acc_iter_time}, the adversarially trained model has higher adversarial accuracy under projected Sinkhorn, but has lower adversarial accuracy under our stronger attacks.

In \Cref{tab:mnist_adv_training_projected_sinkhorn}, the post-processing algorithm does not work quite well with Frank-Wolfe.
Normally, increasing the perturbation budget $\epsilon$ should decrease the adversarial accuracy.
Indeed, if we do not post-process the output, the adversarial accuracy decreases monotonically for Frank-Wolfe.
However, after post-processing, the adversarial accuracy increases a lot, and even increases as the the perturbation budget increases.
For this model, our post-processing algorithm does not seems to be ``compatible'' with Frank-Wolfe.
Doing a better job on optimizing the Wasserstein constrained problem ignoring the hypercube constraint does not necessarily give a good solution to the problem with hypercube constraint. 

\begin{table*}[h]
\centering
\caption{MNIST model adversarially trained by Frank-Wolfe with dual LMO ($\epsilon=0.3$).}
\label{tab:mnist_adv_training_frank_wolfe}
\begin{tabular}{l c c c c c}
\toprule
\multicolumn{1}{c}{\multirow{1}{*}{method}} & $\epsilon = 0.1$ & $\epsilon = 0.2$ & $\epsilon = 0.3$ & $\epsilon = 0.4$ & $\epsilon = 0.5$ \\
\midrule
PGD + Proj. Sink. ($\gamma = 1/ 1000$)  & $94.9$ &  $94.0$ &  $93.0$ &  $91.9$ &  $90.5$ \\
PGD + Proj. Sink. ($\gamma = 1/ 1500$)  & $94.5$ &  $93.2$ &  $91.5$ &  $89.3$ &  $86.8$ \\
PGD + Proj. Sink. ($\gamma = 1/ 2000$)  & $-$    &  $-$    &  $-$    &  $-$    &  $-$    \\
PGD + Dual Proj.                        & $92.6$ &  $87.8$ &  $80.2$ &  $70.7$ &  $59.7$ \\
FW ~ + Dual LMO                         & $92.5$ &  $88.1$ &  $82.1$ &  $75.3$ &  $66.8$ \\
\midrule
PGD + Dual Proj. ~(w/o post-processing) & $91.1$ &  $82.9$ &  $71.3$ &  $58.4$ &  $44.6$ \\
FW ~ + Dual LMO (w/o post-processing)   & $91.1$ &  $83.9$ &  $73.9$ &  $63.0$ &  $50.9$ \\
\bottomrule
\end{tabular}
\end{table*}

\begin{table*}[h]
\centering
\caption{MNIST model adversarially trained by PGD with projected Sinkhorn.}
\label{tab:mnist_adv_training_projected_sinkhorn}
\begin{tabular}{l c c c c c}
\toprule
\multicolumn{1}{c}{\multirow{1}{*}{method}} & $\epsilon = 0.1$ & $\epsilon = 0.2$ & $\epsilon = 0.3$ & $\epsilon = 0.4$ & $\epsilon = 0.5$ \\
\midrule
PGD + Proj. Sink. ($\gamma = 1/ 1000$)  & $95.0$ &  $92.4$ &  $90.5$ &  $88.5$ &  $86.5$ \\
PGD + Proj. Sink. ($\gamma = 1/ 1500$)  & $93.8$ &  $90.0$ &  $82.7$ &  $85.2$ &  $90.5$ \\
PGD + Proj. Sink. ($\gamma = 1/ 2000$)  & $-$    &  $-$    &  $-$    &  $-$    &  $-$    \\
PGD + Dual Proj.                        & $ 1.1$ &  $ 0.8$ &  $ 0.6$ &  $ 0.6$ &  $ 0.6$ \\
FW ~ + Dual LMO                         & $ 4.8$ &  $21.6$ &  $34.5$ &  $39.2$ &  $39.6$ \\
\midrule
PGD + Dual Proj. ~(w/o post-processing) & $ 1.0$ &  $ 0.4$ &  $ 0.3$ &  $ 0.3$ &  $ 0.3$ \\
FW ~ + Dual LMO (w/o post-processing)   & $ 0.8$ &  $ 0.3$ &  $ 0.2$ &  $ 0.1$ &  $ 0.0$ \\
\bottomrule
\end{tabular}
\end{table*}

\clearpage

\subsubsection{CIFAR-10}

We adversarially train a robust model using Frank-Wolfe with dual LMO and a fixed perturbation budget $\epsilon=0.005$.
The inner maximization is approximated with at most $30$ iterations of Frank-Wolfe.
This model achieves $82.57\%$ clean accuracy.
The adversarial accuracy of this model is shown in \Cref{tab:cifar_adv_training_frank_wolfe}.

For comparison, we also present results on attacking an adversarially trained model by PGD with projected Sinkhorn.
We use a pretrained model released by \textcite{WongSK19}, which is adversarially trained using an adaptive perturbation budget $\epsilon \in [0.01, 0.38]$. 
This model achieves $81.68\%$ clean accuracy.

\begin{table*}[h]
\centering
\caption{CIFAR-10 model adversarially trained by Frank-Wolfe with dual LMO ($\epsilon=0.005$)}
\label{tab:cifar_adv_training_frank_wolfe}
\begin{tabular}{l c c c c c}
\toprule
\multicolumn{1}{c}{\multirow{1}{*}{method}} & $\epsilon = 0.001$ & $\epsilon = 0.002$ & $\epsilon = 0.003$ & $\epsilon = 0.004$ & $\epsilon = 0.005$ \\
\midrule
PGD + Proj. Sink. ($\gamma=1 /  3000$)  & $82.4$ & $82.3$ & $82.1$ & $81.9$ & $81.8$ \\
PGD + Proj. Sink. ($\gamma=1 / 10000$)  & $82.0$ & $81.5$ & $81.0$ & $80.6$ & $80.1$ \\
PGD + Proj. Sink. ($\gamma=1 / 20000$)  & $-$ & $-$ & $-$ & $-$ & $-$  \\
PGD + Dual Proj.                        & $76.8$ &  $71.8$ &  $67.0$ &  $62.0$ &  $56.8$ \\
FW ~ + Dual LMO                         & $77.4$ &  $73.7$ &  $70.2$ &  $66.8$ &  $62.6$ \\
\midrule
PGD + Dual Proj. ~(w/o post-processing) & $76.5$ &  $71.3$ &  $66.4$ &  $61.4$ &  $56.2$ \\
FW ~ + Dual LMO (w/o post-processing)   & $77.2$ &  $73.7$ &  $70.4$ &  $67.3$ &  $63.9$ \\
\bottomrule
\end{tabular}
\end{table*}

\begin{table}[h]
\centering
\caption{CIFAR-10 model adversarially trained by PGD with projected Sinkhorn.}
\label{tab:cifar_adv_training_projected_sinkhorn}
\begin{tabular}{l c c c c c}
\toprule
\multicolumn{1}{c}{\multirow{1}{*}{method}} & $\epsilon = 0.001$ & $\epsilon = 0.002$ & $\epsilon = 0.003$ & $\epsilon = 0.004$ & $\epsilon = 0.005$ \\
\midrule
PGD + Proj. Sink. ($\gamma=1 /  3000$)  & $81.7$ & $81.6$ & $81.6$ & $81.6$ & $81.5$ \\
PGD + Proj. Sink. ($\gamma=1 / 10000$)  & $81.6$ & $81.4$ & $81.3$ & $81.2$ & $81.0$ \\
PGD + Proj. Sink. ($\gamma=1 / 20000$)  & $-$    & $-$    & $-$    & $-$    & $-$    \\
PGD + Dual Proj.                        & $72.5$ &  $64.4$ &  $56.3$ &  $48.8$ &  $42.2$ \\
FW ~ + Dual LMO                         & $72.4$ &  $64.0$ &  $55.5$ &  $47.8$ &  $41.2$ \\
\midrule
PGD + Dual Proj. ~(w/o post-processing)  & $72.0$ &  $63.3$ &  $54.9$ &  $47.2$ &  $40.6$ \\
FW ~ + Dual LMO (w/o post-processing)     & $71.7$ &  $62.2$ &  $52.9$ &  $44.7$ &  $37.5$ \\
\bottomrule
\end{tabular}
\end{table}

\clearpage

\section{Post-processing: Capacity Constrained Projection for Hypercube Constraint}
\label{sec:appendix_post_processing}


In this section, we present the post-processing algorithm mentioned in \S\ref{sec:hypercube}.

\subsection{Algorithm}

Suppose that pixel values of input images are represented by real numbers in $[0, 1]^n$. We need to add one additional constraint $\Pi^\top \one \leq \one$, which results in the following Euclidean projection problem:
\begin{align}
\begin{split}
    & \mini_{\Pi \geq 0} ~~ \frac{1}{2} \| \Pi - G \|_\mathrm{F}^2 \\
    & \sbjto ~ \Pi \one = \xv, ~ \Pi^\top \one \leq \one, \langle \Pi, C \rangle \leq \delta.
\end{split}
\end{align}
We call this problem a \emph{capacity constrained projection}, since the additional constraint essentially specifies the maximum mass that a pixel location can receive. We introduce the partial Lagrangian to derive the following dual problem:
\begin{align*}
    \maxi_{\lambda \geq 0,  \muv \geq 0} ~ g(\lambda, \muv),
\end{align*}
where
\begin{align*}
    g(\lambda, \muv) = \min_{\Pi \geq 0, \Pi \one = \xv} ~ & \frac{1}{2} \| \Pi - G \|_\mathrm{F}^2 + \lambda \left(\langle \Pi, C \rangle - \delta\right) + \muv^\top\left(\Pi^\top \one - \one\right) \\
                     = \min_{\Pi \geq 0, \Pi \one = \xv} ~ & \frac{1}{2} \| \Pi - G \|_\mathrm{F}^2 + \lambda \left(\langle \Pi, C \rangle - \delta\right) + \langle \Pi, \one \muv^\top \rangle - \muv^\top \one \\
                     = \min_{\Pi \geq 0, \Pi \one = \xv} ~ & \frac{1}{2} \| \Pi - G + \lambda C + \one \muv^\top \|_\mathrm{F}^2 - \langle G, \lambda C + \one \muv^\top \rangle + \frac12 \| \lambda C + \one \muv^\top \|_{\mathrm{F}}^2 - \lambda \delta - \muv^\top \one
\end{align*}
We optimize $g(\lambda, \muv)$ by alternating maximization on $\lambda$ and $\muv$ respectively:
fixing $\muv$, we maximize $\lambda$ via bisection method; fixing $\lambda$, we maximize $\muv$ via $k$ steps gradient ascent with nesterov acceleration, where $k$ is a hyperparameter.
Evaluating the gradient of $\muv$ in bisection method, as well as evaluating the gradient of $\muv$, can be reduced to simplex projections following similar derivations in dual projection (\S\ref{sec:dual_proj}).

Due to sublinear convergence rate of gradient ascent, it may be slow to obtain a very high precision solution.
However, empirical evidences show that a few hundred alternating maximization (with $k$ around $10$ or $20$) converges to a solution with reasonable precision (\eg satisfying hypercube constraint up to the third digit after the decimal point), and convergence to these modest precision solutions is already sufficient for generating valid Wasserstein adversarial examples.

\section{Additional Related Work}


In this section, we comment on a few more works related to threat models beyond the standard $\ell_p$ metric.
A number of authors have recently explored geometric transformations as an adversarial attack against deep models.
As already mentioned in the main paper, \citeadd{EngstromTTSM19} studied adversarial rotations and translations where perturbation is measured by the degree of rotation and translation.
Similarly, \citetadd{AlaifariAG19} \citetadd{AlaifariAG19} considered adversarial deformations and used the maximum Euclidean size of the deformation as the perturbation budget.
\citetadd{KanbakMF18} studied the more general case where spatial transformation is parameterized as a Lie group, and employed the geodesic distance in the image appearance manifold to measure  perturbation size. \citetadd{XiaoZLHLS18}, on the other hand, modeled spatial transformations as a vector flow and used the total variation of the flow to measure  perturbation size.
On a high level, spatial transformation shares some similarity with the Wasserstein threat model, as they both involve pixel mass movement.
In fact, the Wasserstein threat model can be treated as a further relaxation of spatial transformations, where we are not only allowed to move pixels but also to change pixel values (\eg pixel mass splitting).
In this sense, the Wasserstein threat model is a combination of spatial transformation and the standard $\ell_p$ additive perturbation, although the way it measures perturbation is entirely different from either.
All of the above methods employ first order gradient algorithms to solve their respective optimization problems.
We note that some authors, \eg \citetadd{AthalyeEIK18,LiSK19}, have already considered physical attack for real world objects.

Lastly, we mention that certification algorithms for spatial transformations and the Wasserstein threat model have recently been developed by \citetadd{BalunovicBSGV19} and \citetadd{LevineFeizi19}, respectively.  

\bibliographyadd{ref}
\bibliographystyleadd{icml2020}

\end{document}